\pdfoutput=1
\documentclass[11pt,letterpaper]{article}

\usepackage[T1]{fontenc}
\usepackage[margin=1in]{geometry}
\usepackage{newtxtext}
\usepackage{helvet}
\usepackage{courier}
\usepackage[hyphens]{url}
\usepackage{graphicx}
\usepackage{natbib}
\usepackage{caption}
\usepackage{placeins}
\usepackage{xcolor}
\usepackage{algorithm}
\usepackage{algorithmic}

\usepackage{newfloat}
\usepackage{listings}
\DeclareCaptionStyle{ruled}{labelfont=normalfont,labelsep=colon,strut=off}
\floatstyle{ruled}
\newfloat{listing}{tb}{lst}{}
\floatname{listing}{Listing}

\usepackage{amsmath,amssymb,amsthm,booktabs}
\theoremstyle{plain}
\newtheorem{theorem}{Theorem}[section]
\newtheorem{lemma}[theorem]{Lemma}
\newtheorem{proposition}[theorem]{Proposition}
\newtheorem{corollary}[theorem]{Corollary}

\theoremstyle{definition}
\newtheorem{definition}[theorem]{Definition}

\theoremstyle{remark}
\newtheorem{remark}[theorem]{Remark}
\usepackage[hidelinks,bookmarks=false]{hyperref}

\title{How Inefficient Is Natural Gradient Descent? From Exact Optimality to $\Theta(\sqrt{\log d})$ Divergence}
\author{Guni Sharon\thanks{Department of Computer Science and Engineering,
Texas A\&M University, College Station, TX 77843, USA. E-mail:
\href{mailto:guni@tamu.edu}{guni@tamu.edu}}
\and
Alan Kuhnle\thanks{Department of Computer Science and Engineering,
Texas A\&M University, College Station, TX 77843, USA. E-mail:
\href{mailto:kuhnle@tamu.edu}{kuhnle@tamu.edu}}}
\date{}

\begin{document}

\newcommand{\COMMENT}[1]{}

\maketitle

\begin{abstract}
Natural gradient descent (NGD) underlies common methods in ML. For dually flat families, idealized NGD on the forward Kullback--Leibler objective follows the mixture geodesic which is often longer than the shortest Fisher--Rao path. We quantify this overhead by the inefficiency ratio \(R \ge 1\), the Fisher length of the mixture geodesic divided by the Fisher--Rao distance, and bound its supremum over endpoint pairs as a function of the parameter dimension \(d\). A tensor criterion identifies the regime (I) families, with \(R=1\) everywhere: exactly those with quadratic potential or dimension one, such as fixed-covariance Gaussians. For non-quadratic families, we prove two further regimes: (II) bounded third-order skewness plus finite Fisher--Rao diameter yields a dimension-independent bound; and (III) for products of scale families---including Gaussian covariances and Gamma rates---\(R\) grows as \(\Theta(\sqrt{\log d})\), unbounded in \(d\). Under a per-step Fisher-chord budget, \(R\) translates to a practical computational cost: NGD requires asymptotically at least \(R\) times as many steps as an optimizer following the Fisher--Rao geodesic. Experiments confirm all three regimes: \(R=1\) to machine precision for quadratic-potential families (I), the categorical bound \(\pi/(2\sqrt{2})\) is approached but not attained (II), and sampled scale-product \(R\) grows with \(d\), reaching \(R \approx 1.5\) for long, high-dimensional moves (III).
\end{abstract}

\section{Introduction}
\label{sec:intro}

Natural gradient descent (NGD) is a common optimizer choice across modern learning
systems, including reinforcement learning \cite{kakade2001natural,schulman2015trust}, variational inference \cite{khan2018fast}, and evolution strategies
\cite{wierstra2014natural}. NGD is often credited with faster and more stable progress in distribution spaces because it replaces the Euclidean gradient with the Fisher-metric
gradient, making its updates invariant to parameterization
\cite{AmariNG,amari2016information,martens2020new}. In practice, however, NGD is expensive and fragile in four well-documented ways: (1) \emph{computational cost}---for a model with $d$ parameters (the dimension of the natural parameter $\theta$) the Fisher information matrix (FIM) is $d\times d$, requiring $\mathcal{O}(d^3)$ compute to solve; (2)
\emph{ill-conditioning} \cite{martens2020new}---overparameterized models routinely produce a
singular or near-singular FIM; (3) \emph{noisy empirical curvature} \cite{kunstner2019limitations}---the FIM used in
practice is a sample-based estimate, not the exact metric; and (4) \emph{saddle-point sensitivity}---second-order preconditioning changes how a
trajectory interacts with saddle points \cite{dauphin2014identifying}. These four difficulties have each
been studied extensively. A fifth, more basic question has not (to our knowledge): is the NGD trajectory itself efficient? We show that the answer is \textbf{no} in general. The reason is the Amari--Chentsov skewness tensor, which measures how the Fisher metric changes across the model. Idealized NGD follows the mixture geodesic---a path that is straight in expectation coordinates---and this path need not track the shortest route through the Fisher geometry.

We quantify the excess length of the idealized NGD path by a single scale-free number, the \emph{Fisher inefficiency ratio}
\[
R(p,q)=\frac{L(\gamma^{(m)})}{\delta_{\mathrm{geo}}(p,q)}\ \ge\ 1,
\]
the Riemannian length of the mixture geodesic, $L(\gamma^{(m)})$, divided by the Fisher--Rao
distance $\delta_{\mathrm{geo}}$ between its endpoints $(p,q)$.
Equality $R(p,q)=1$ means NGD's
idealized trajectory is length-optimal;
$R(p,q)>1$ means that NGD follows a suboptimal path. We study the worst-case ratio $R^\ast(M)=\sup_{p,q}R(p,q)$ of a family $M$, and how it behaves as
the family dimension grows.

\paragraph{What makes $R$ exceed 1?} We show that NGD's optimality ($R=1$) is determined by the \textit{Amari--Chentsov tensor} $T$, a fully symmetric $(0,3)$-tensor that measures the gap between the two dual affine connections of a dually flat family. 
We prove that for $R=1$, the Amari--Chentsov tensor $T$, contracted twice against the $m$-geodesic's velocity, must be proportional to that velocity at every point along the path (Theorem~\ref{thm:pointwise}). Using this criterion, we separate families into three demonstrated regimes:
\begin{itemize}
\item[(I)] \emph{Fisher efficient} ($R^\ast(M)=1$): those with $d=1$ and/or quadratic potential, e.g., fixed-covariance Gaussians.
\item[(II)] \emph{Uniformly bounded} ($\sup_d R_d^\ast<\infty$): we show that bounded third-order skewness of the Amari--Chentsov tensor together with finite Fisher--Rao diameter forces a bound on $R_d^\ast$ that does not depend on the dimension.
\item[(III)] \emph{Dimensionally unbounded} ($\sup_d R_d^\ast=\infty$): for products of scale families, including Gaussian covariances and Gamma rates, we prove $R_d^\ast=\Theta(\sqrt{\log d})$.
\end{itemize}

\paragraph{Operational meaning.} We show that the ratio $R$ acquires a concrete interpretation under a per-step Fisher-chord budget. If each update's expectation-coordinate displacement is capped at $\varepsilon$ in the local Fisher norm, then the number of steps along a smooth path is asymptotic to its length divided by $\varepsilon$. Under this budget, idealized NGD requires asymptotically at least $R$ times as many steps as a Fisher--Rao geodesic oracle (Proposition~\ref{prop:steps}). Concretely, $R\approx1.5$ means at least $50\%$ more gradient steps.

\paragraph{Experiments.} We sample random endpoint pairs and compute $R$ numerically. The Fisher-efficient regime (I) returns $R=1$ to machine precision. For Regime II we show that for the categorical family, $R$ stays strictly below the theoretical cap, and the observed maxima approach it while remaining nearly constant across dimensions. For scale-product families (Regime III), sampled $R$ grows with dimension, reaching $R\approx1.5$ for long, high-dimensional covariance and Gamma moves, consistent with the $\Theta(\sqrt{\log d})$ rate. 

The main contributions are the pointwise tensor criterion
(Theorem~\ref{thm:pointwise}), the dimension-free skewness--diameter bound
(Theorem~\ref{thm:bounded}), an exact categorical constant
(Theorem~\ref{thm:categorical}), the $\Theta(\sqrt{\log d})$ scale-product rate
(Theorem~\ref{thm:class3}), and the one-sided step-count interpretation
(Proposition~\ref{prop:steps}). 

\section{Setup and Scope}
\label{sec:setup}

We first introduce exponential families and the
dually flat geometry they carry (Section~\ref{sec:dually-flat}), then define
the inefficiency ratio $R$ and fix the scope (Section~\ref{sec:ratio}).
A full glossary is provided in Appendix~\ref{apx:sec:glossary}.

\subsection{Exponential families and dually flat geometry}
\label{sec:dually-flat}

Every family studied in this paper is an exponential family. In its canonical
form,
\[
p_\theta(x)=\exp\bigl(\langle\theta,t(x)\rangle-\psi(\theta)\bigr),
\]
where $\theta\in\mathbb{R}^d$ is the natural parameter, $t(x)$ the sufficient
statistic, and $\psi$ the log-partition function---a strictly convex potential.
The expectation parameter is
$\eta=\mathbb{E}_\theta[t(x)]=\nabla\psi(\theta)$,
and the Fisher information matrix is the Hessian
$g=\nabla^2\psi$.

This structure carries a dually flat geometry. The manifold
$M=\{p_\theta\}$ is equipped with the Fisher metric $g$ and two flat,
torsion-free affine connections $\nabla^{(m)},\nabla^{(e)}$ that are mutually
dual with respect to $g$ and admit global affine charts $\eta,\theta$
related through the convex potential $\psi$ by $\eta=\nabla\psi(\theta)$
\cite{Amari2000MethodsOI, amari2016information, shima2007geometry}. The
mixture connection $\nabla^{(m)}$ has vanishing Christoffel symbols in
$\eta$-coordinates; the exponential connection $\nabla^{(e)}$ has vanishing
Christoffel symbols in $\theta$-coordinates.
Appendix~\ref{apx:sec:facts} collects the standard facts (F1--F5) invoked
throughout.

\begin{definition}[Dual geodesics and the Amari--Chentsov tensor]
The \emph{$m$-geodesic} $\gamma^{(m)}:[0,1]\to M$ is the unique curve satisfying
\[
\eta(\gamma^{(m)}(t))=(1-t)\eta(p)+t\eta(q),
\]
the affine segment in expectation coordinates; the \emph{$e$-geodesic}
$\gamma^{(e)}$ is the analogous affine segment in natural coordinates $\theta$.
Neither is in general the shortest path under $g$
\cite{nielsen2020elementary, nielsen2021geodesic}.

The Levi--Civita connection is the average
$\nabla^{\mathrm{LC}}=\tfrac12(\nabla^{(m)}+\nabla^{(e)})$. The gap between the
two dual connections is measured by the \emph{Amari--Chentsov tensor} $T$,
defined by pairing their difference with a third vector via the metric:
\[
T(X,Y,Z)=g\bigl(\nabla^{(e)}_X Y-\nabla^{(m)}_X Y,\,Z\bigr).
\]
In words: to evaluate $T(X,Y,Z)$, differentiate the vector field $Y$ in
direction $X$ using each connection, take their difference, and measure its
Fisher inner product with $Z$. Then, $T$ is a fully symmetric $(0,3)$-tensor; in the
natural chart it is the negative of the third derivative of the potential,
$T_{ijk}=-\partial_i\partial_j\partial_k\psi$, the negative third cumulant of the
sufficient statistic. This $T$ appears in prior analyses of gradient flows and
pregeodesics
\cite{fujiwara1995gradient, bravetti2025asymmetric, bravetti2026gradient}.
\end{definition}

\begin{definition}[Scale family]\label{def:scale-family}
A \emph{scale family} is a one-dimensional dually flat manifold with
expectation coordinate $\eta\in(0,\infty)$ and Fisher metric
$g_{\eta\eta}=\kappa/\eta^2$ for a constant $\kappa>0$.
\end{definition}

\begin{definition}[Categorical family]\label{def:categorical-family}
The \emph{categorical family} on $k$ outcomes is the exponential family of
probability distributions on $\{1,\ldots,k\}$, identified with the interior of
the $(k-1)$-simplex $\Delta_{k-1} = \{p \in \mathbb{R}^k : p_i > 0,\ \sum_{i=1}^k p_i = 1\}$.
\end{definition}

\subsection{The Fisher inefficiency ratio}
\label{sec:ratio}

For a given family and endpoint pair ($p,q$), we quantify the
inefficiency of NGD by the \emph{Fisher inefficiency ratio}. This ratio is an
instance of a standard geometric quantity, the length-to-distance
\emph{dilation} of a prescribed curve, studied more generally by Nesterov and
Todd~\cite[Definition~3.3]{nesterov2002riemannian}. Our contribution is to
analyze this dilation for Fisher--Rao mixture geodesics specifically,
including exact family-wise constants, dimension-dependent rates, and a
restricted natural-gradient interpretation; Appendix~\ref{apx:sec:related-work}
situates this contribution fully within that and other prior work.

\begin{definition}[Inefficiency ratio]
For distinct $p,q\in M$,
\[
R(p,q)=\frac{L\big(\gamma^{(m)}\big)}{\delta_{\mathrm{geo}}(p,q)},
\]
where $L(\gamma^{(m)})$ is the Riemannian length of the $m$-geodesic and
$\delta_{\mathrm{geo}}$ is the minimal Riemannian distance (corresponding to
the Fisher--Rao geodesic).
\end{definition}

Because $\delta_{\mathrm{geo}}(p,q)$ is the infimal length over curves
connecting $p$ and $q$, we have $R(p,q)\ge 1$, with equality exactly when the
image of $\gamma^{(m)}$ is a minimizing (Fisher--Rao) geodesic.
$R(p,q)$ is invariant under reparametrization: it depends only on the image of
the $m$-geodesic and its endpoints, not on the speed at which that image is
traversed (the rate $\|\dot\gamma\|_g$, which reparametrization is free to
change).

\paragraph{Scope.} Throughout, $M$ is a connected, globally dually flat manifold
with Fisher metric $g$ and fixed dual connections
$\nabla^{(m)},\nabla^{(e)}$. Note that without a fixed dual-flat structure our
theoretical framework does not apply.
\section{Fisher-Inefficiency Regimes}
\label{sec:classes}

We now classify dually flat families according to their worst-case inefficiency across dimensions. For a single manifold $M$, define the worst-case Fisher inefficiency as $R^\ast(M) = \sup_{p,q \in M} R(p,q)$. Fix a family $\mathcal{F}$ of statistical manifolds. For each dimension $d$, let $R_d^\ast(\mathcal{F}) = \sup \{\,R^\ast(M) : M \in \mathcal{F},\ \dim(M) = d\,\}$, the worst-case inefficiency across all $d$-dimensional members of $\mathcal{F}$. When $\mathcal{F}$ is clear from context we write $R_d^\ast$. Three regimes emerge: (I) Fisher efficient: $R_d^\ast = 1$ for all $d$; (II) uniformly bounded inefficiency: $1 < \sup_d R_d^\ast < \infty$; (III) dimensionally unbounded: $\sup_d R_d^\ast = \infty$.

We use the Amari--Chentsov tensor as the local efficiency (Regime I) obstruction.
Theorem~\ref{thm:pointwise} gives a pointwise criterion for when a single
$m$-geodesic attains $R=1$; Theorem~\ref{thm:classI} shows that demanding
$R^\ast(M)=1$ for a manifold $M$ forces $T\equiv0$ and a quadratic potential on $M$. For families
with non-vanishing $T$, Theorem~\ref{thm:bounded} gives a sufficient condition
for uniform boundedness---finite $g$-normalized skewness $\kappa_T$ together
with finite Fisher--Rao diameter forces $R^\ast\le e^{\kappa_T D/2}$. The
scale-product examples of Section~\ref{sec:class3} demonstrate that the
diameter hypothesis cannot simply be discarded from this sufficient theorem,
even when the normalized skewness is bounded.

\subsection{The pointwise optimality criterion}
\label{sec:pointwise}

A single tensor condition decides whether a
given $m$-geodesic is a Levi--Civita pregeodesic. This detects, for a specific
endpoint pair, whether the mixture-straight path is also Fisher-shortest.

\begin{theorem}[Pointwise pregeodesy criterion]\label{thm:pointwise}
Let $p,q\in M$ be distinct and let $\dot\gamma$ be the velocity of their
$m$-geodesic $\gamma$. Then $\gamma$ is a Levi--Civita pregeodesic if and only
if there is a scalar field $\mu(s)$ along $\gamma$ with
\[
T(\dot\gamma,\dot\gamma,Z)=\mu\,g(\dot\gamma,Z)\qquad\text{for all }Z.
\]
Consequently $R(p,q)=1$ implies the condition; and the condition implies
$R(p,q)=1$ whenever the image of $\gamma$ is globally length-minimizing---for
example, if $M$ is Hadamard or the image lies in a strongly convex normal
neighborhood.
\end{theorem}

\paragraph{Intuition.} Along an $m$-geodesic the mixture acceleration vanishes,
so the Levi--Civita acceleration is $\tfrac12\nabla^{(e)}_{\dot\gamma}\dot\gamma$.
Pairing with an arbitrary vector field $Z$ and using $T=\nabla^{(e)}-\nabla^{(m)}$
expresses this acceleration as the covector $\tfrac12 T(\dot\gamma,\dot\gamma,\cdot)$.
Pregeodesy demands that this covector be everywhere proportional to
$g(\dot\gamma,\cdot)$---exactly the displayed identity with $\mu=2\tilde\mu$.
If the curve attains the Riemannian distance ($R=1$), it must be a minimizing
geodesic up to reparametrization, forcing the condition. Conversely, when the
condition holds and the curve's image is globally length-minimizing, the
pregeodesic segment is a minimizing geodesic and $R=1$. Full proof:
Appendix~\ref{apx:sec:classI}.

\subsection{Regime I: Fisher efficient ($R^\ast=1$)}
\label{sec:class1}

The Fisher-efficient
regime is characterized intrinsically: $R^\ast=1$ everywhere exactly when the
Amari--Chentsov tensor vanishes. The one-dimensional case is an exception---there,
the interval topology forces optimality regardless of $T$.

\begin{lemma}[Tensor lemma]\label{lem:tensor}
Let $A$ be a fully symmetric $(0,3)$ tensor on a Riemannian manifold of dimension
$\ge2$. If for every tangent vector $v$ there is a scalar $\lambda(v)$ with
$A(v,v,Z)=\lambda(v)\,g(v,Z)$ for all $Z$, then $A\equiv0$.
\end{lemma}

\begin{lemma}[Dimension one]\label{lem:dim-one}
Let $\mathcal{F}$ be a one-dimensional dually flat family whose global expectation
image is an interval. Then $R^\ast(\mathcal{F}) = 1$, regardless of $T$.
\end{lemma}

Proofs of both lemmas: Appendix~\ref{apx:sec:classI}.

\begin{theorem}[Fisher-efficient characterization]\label{thm:classI}
Assume the global expectation-coordinate image $\eta(M)$ is convex, so every
affine segment between two points remains in the model. Then $R^\ast(M)=1$ if
and only if $\dim M=1$ or $\psi$ is quadratic (equivalently, $\eta$ and
$\theta$ are affinely related).
\end{theorem}

\paragraph{Intuition.} If $\psi$ is quadratic then $\eta=\nabla\psi(\theta)$ is affine,
so $g$ is constant and $\nabla^{(m)}=\nabla^{(e)}=\nabla^{\mathrm{LC}}$. The
$\eta$-affine segment is a Euclidean straight line in a flat metric---hence
globally minimizing, giving $R^\ast=1$. Conversely, if $R^\ast=1$ and $\dim\ge2$,
Theorem~\ref{thm:pointwise} forces every $m$-geodesic to satisfy the pointwise
criterion. Fixing $p$ and letting $q\to p$, the initial velocities cover
$T_pM$, so the condition holds for all directions; Lemma~\ref{lem:tensor} then
forces $T\equiv0$, hence $\psi$ is quadratic. The one-dimensional case is
genuinely special (Lemma~\ref{lem:dim-one}): the $\eta$-interval is traversed
monotonically, and any back-and-forth only adds length, so $R^\ast=1$ even when
$T\neq0$. Full proof: Appendix~\ref{apx:sec:classI}.

\paragraph{Examples.} The Gaussian location family $N(\mu,\Sigma_0)$ at fixed
$\Sigma_0$ has $\theta=\Sigma_0^{-1}\mu$ and $\eta=\mu$ affinely related, so
$\psi$ is quadratic and $R\equiv1$: natural gradient on the mean of a
fixed-covariance Gaussian is globally Fisher-optimal
\cite{skovgaard1984riemannian}. Linear-Gaussian models share this affine
relation. Every one-parameter exponential family is Fisher efficient by
Lemma~\ref{lem:dim-one}, even when $T\neq0$. In dimension at least two, a
non-quadratic potential (under the convex-domain hypothesis) yields $R>1$
somewhere.

\subsection{Regime II: Uniformly bounded inefficiency ($R_d^\ast=O(1)$)}
\label{sec:class2}

We now give a sufficient condition for uniformly bounded inefficiency: finite normalized skewness plus finite Fisher--Rao diameter. The mechanism behind Theorem~\ref{thm:bounded} is controlled
by the Amari--Chentsov tensor itself, normalized by the
metric. Consequently we define the \emph{$g$-normalized skewness}
\[
\kappa_T(M)\;=\;\sup_{x\in M}\ \sup_{v\ne0}\
\frac{|T_x(v,v,v)|}{g_x(v,v)^{3/2}},
\]
a dimension-free scalar invariant of the family.

\begin{theorem}[Bounded inefficiency from skewness and diameter]\label{thm:bounded}
Let $\dim M\ge2$, and let $D(M)=\sup_{p,q\in M}\delta_{\mathrm{geo}}(p,q)$ be
the Fisher--Rao diameter. If $\kappa_T(M)<\infty$ and $D(M)<\infty$, then
\[
R^\ast(M)\;\le\;\exp\!\Big(\frac{\kappa_T(M)\,D(M)}{2}\Big),
\]
independently of $\dim M$. If moreover the global expectation-coordinate image
is convex and $\psi$ is non-quadratic, then $R^\ast(M)>1$.
\end{theorem}

\paragraph{Intuition.} The proof uses a measuring-tape argument. Fix one
expectation-coordinate direction as a ruler and carry it across the
manifold; the skewness $\kappa_T$ controls how fast its Fisher length can
change, so between two points the ruler distorts by at most
$e^{\kappa_T\delta/2}$, where $\delta$ is their Fisher--Rao distance. Freeze
the norm at the point $z$ on the $m$-geodesic where the ruler
$\eta(q)-\eta(p)$ is longest: the $m$-geodesic's length is bounded above by
this frozen norm, while every competing curve, by the same ruler bound and
the finite diameter, has length at least $e^{-\kappa_T D/2}$ times the frozen
norm. The ratio is $e^{\kappa_T D/2}$. Crucially, this argument compares the
\emph{same fixed} coordinate direction at different points. Formally, this uses a
polarization bound (Lemma~\ref{app:lem-polar}) to convert the skewness
hypothesis into $|T(\dot c,V_a,V_a)|\le\kappa_T\|\dot c\|\,g(V_a,V_a)$ for
any fixed $\eta$-constant vector field $V_a$, and a ruler-distortion lemma
(Lemma~\ref{app:lem-ruler}) to integrate this bound along curves. Full proof: Appendix~\ref{apx:sec:classII}.

\paragraph{Scope of the skewness--diameter criterion.}
\label{rem:skewness-scope}
In coordinates, the skewness hypothesis ($\kappa_T(M)<\infty$) is the normalized third-derivative
bound
\[
|D^3\psi_x[v,v,v]|
\le \kappa\bigl(D^2\psi_x[v,v]\bigr)^{3/2}.
\]
Together with finite Fisher--Rao diameter, this yields
Theorem~\ref{thm:bounded}. The criterion is sufficient, not necessary:
scale products show that finite diameter matters, while the categorical family
has a bounded ratio even though its normalized skewness diverges at the
simplex boundary. Appendix~\ref{apx:sec:related-work} relates the displayed
inequality to the self-concordance inequality of Nesterov and
Nemirovskii~\cite{nesterov1994interior} and to the Hessian-manifold dilation
studied by Nesterov and Todd~\cite{nesterov2002riemannian}.

\paragraph{Categorical: an exact constant.}
The square-root map $p\mapsto2\sqrt p$ identifies the
categorical family with the positive orthant of a radius-$2$ sphere
\cite{nielsen2023simple,nielsen2024approximation}. On the sphere,
$d_{\mathrm{FR}}(p,q)=2\arccos\sum_i\sqrt{p_iq_i}$ (great-circle arc), of
diameter $\pi$, while the $m$-geodesic is the image of ordinary linear
interpolation in the simplex. Its image is the arc of a smaller circle: for a
move confined to two coordinates of total mass $m$, both endpoints lie on the
circle of radius $2\sqrt m$, and the $m$-geodesic traces that circle's own arc
rather than cutting across the ambient sphere. As the moved mass vanishes, the
ambient geodesic distance converges to that small circle's own straight
chord, so $R$ converges to the small circle's intrinsic arc-to-chord
ratio---a quantity $\ge1$, increasing in the swept angle, and maximized at a
quarter turn ($\psi=\pi/2$), giving $R\to\pi/(2\sqrt2)$. The upper bound uses an arcsine-weight integral
identity, one Cauchy--Schwarz step, and $\sin x\le x$; the lower bound is an
explicit two-outcome family.

\begin{theorem}[Categorical constant]\label{thm:categorical}
For the categorical family on $k\ge3$ outcomes,
$R^\ast(\Delta_{k-1})=\pi/(2\sqrt2)\approx1.1107$, independent of $k$ and
approached but not attained; for $k=2$, $R^\ast(\Delta_1)=1$.
\end{theorem}

\paragraph{Why the bound is sharp.} Restrict to moves changing only coordinates
$1,2$, freezing the rest; the conserved mass is $m=p_1+p_2=q_1+q_2$. On the
sphere the moving pair traces a circular arc of radius $2\sqrt m$ through
angle $\psi=|\phi_q-\phi_p|$. The ratio
$R_2(m,\psi)=\sqrt m\,\psi/\arccos(1-2m\sin^2(\psi/2))$ is decreasing in $m$
and increasing in $\psi$. As $m\to0$ and $\psi\to\pi/2$,
$R_2\to(\psi/2)/\sin(\psi/2)\to\pi/(2\sqrt2)$. This two-outcome family embeds
in every $\Delta_{k-1}$ ($k\ge3$), so the constant is approached but never
attained.\footnote{The chord-bound equality class of
Appendix~\ref{apx:sec:categorical} is larger than the two-outcome family: every
\emph{pure transfer}---each moving coordinate either drained to zero or filled
from zero---is chord-extremal, and $R$ approaches $\pi/(2\sqrt2)$ only as the
moved mass vanishes.} Full proof: Appendix~\ref{apx:sec:categorical}.

\subsection{Regime III: Dimensionally unbounded inefficiency ($\sup_d R_d^\ast=\infty$)}
\label{sec:class3}

Scale-product families are the canonical example of Regime III: they
satisfy one hypothesis of Theorem~\ref{thm:bounded} (bounded normalized
skewness) but violate the other (finite diameter), and as we prove below,
their inefficiency grows without bound in dimension.  A one-dimensional scale
family (Definition~\ref{def:scale-family}) has metric
$g_{\eta\eta}=\kappa/\eta^2$ and range $(0,\infty)$, giving infinite
Fisher--Rao diameter.  A single factor is Fisher efficient
(Lemma~\ref{lem:dim-one}), so the unbounded growth is a genuinely
multidimensional phenomenon.

A product of $n\ge2$ identical scale families is dually nontrivial
($T\neq0$, so $R^\ast(M_n)>1$ by Theorem~\ref{thm:classI}) and of infinite
diameter.  Its normalized skewness is bounded: $\kappa_T$ is scale-invariant
per factor and tensorizes to the maximum over factors
(Appendix~\ref{apx:sec:classII}).  Of the two hypotheses of
Theorem~\ref{thm:bounded}, it is the diameter hypothesis that fails.
We show that the resulting inefficiency grows as $\Theta(\sqrt{\log n})$
in the number of factors, hence $\sup_d R_d^\ast=\infty$, placing
scale products firmly in Regime III.

\begin{lemma}[Scale-product reduction]\label{lem:scale-reduction}
Let $M_n$ be a product of $n$ one-dimensional factors with expectation
coordinate $\eta_i\in(0,\infty)$ and metric
$g_{\eta_i\eta_i}=\kappa/\eta_i^2$, where $\kappa>0$ is common to the factors.
For endpoints $\eta^p,\eta^q$, set
\[
w_i=\log\frac{\eta_i^q}{\eta_i^p},\qquad
c_i=e^{w_i}-1,\qquad
u_i(s)=\frac{c_i}{1+s c_i},
\]
with $u_i\equiv0$ when $c_i=0$. Then
\[
\begin{aligned}
\delta_{\mathrm{geo}}(p,q)&=\sqrt\kappa\,\|w\|_2,\\
L(\gamma^{(m)})&=\sqrt\kappa\int_0^1\|U(s)\|_2\,ds,\\
R(w)&=\frac{\int_0^1\|U(s)\|_2\,ds}{\|w\|_2}.
\end{aligned}
\]
with $U=(u_i)$. The reduced functional $R(w)$ is independent of $\kappa$.
\end{lemma}

\paragraph{Intuition.} The coordinate change $y_i=\sqrt\kappa\,\log\eta_i$ makes each factor Euclidean,
giving the distance; along the mixture segment $\dot\eta_i/\eta_i=u_i(s)$,
giving the length. Full proof: Appendix~\ref{apx:sec:classIII}.

\paragraph{Scale separation drives the rate.} Divergence comes from coordinates
whose scale parameters $\eta_i$ span many distinct orders of magnitude, not
merely from having many coordinates. Coordinates at well-separated log-scales
contribute additively to the path length but quadratically to the Euclidean
distance $\|w\|_2$ via the Pythagorean theorem. When scales span $m$ decades,
the length scales as $m$ while the distance scales as $\sqrt{m}$, yielding
$R\sim\sqrt{m}\sim\sqrt{\log d}$. Matching upper and lower bounds in
Appendix~\ref{apx:sec:classIII} make this $\Theta(\sqrt{\log d})$.

\begin{theorem}[Scale-product rate]\label{thm:class3}
Let $M_n$ be a product of $n$ identical one-dimensional scale families. Then
$R^\ast(M_n)=\Theta(\sqrt{\log n})$; in particular $\sup_n R^\ast(M_n)=\infty$.
\end{theorem}

\paragraph{Intuition.} By Lemma~\ref{lem:scale-reduction},
$R^\ast(M_n)=\sup_w R(w)$. The upper bound $O(\sqrt{\log n})$ splits
coordinates into small ($|w_i|<1$, handled by $\ell_1$--$\ell_2$ comparison)
and large (handled by a weighted Hardy inequality). The lower bound
$\Omega(\sqrt{\log n})$ places $m=\Theta(\log n)$ blocks at exponentially
separated scales $2,4,\dots,2^m$; within each block's temporal window
$\phi_i(t)$ is uniformly bounded below, and summing over the $m$ disjoint
windows gives $R\ge(\beta/2)\sqrt{m}$. Full proof:
Appendix~\ref{apx:sec:classIII}.

\paragraph{Covariance manifold.}
For zero-mean Gaussians $N(0,\Sigma)$ with the affine-invariant metric, the
potential $-\tfrac12\log\det\Sigma^{-1}$ is non-quadratic, so $\mathrm{SPD}(n)$ is
Hadamard and of infinite diameter \cite{skovgaard1984riemannian,
bhatia2009positive, pennec2020manifold}. Congruence reduction
(Appendix~\ref{apx:sec:classIII}) sends the log relative spectrum to
$w\in\mathbb{R}^n$ and the $\Sigma$-linear ($m$-)geodesic to the functional
above, so $R^\ast(\mathrm{SPD}(n))=\sup_w R(w)=\Theta(\sqrt{\log d})$ by
Theorem~\ref{thm:class3}. This arithmetic-mean path is distinct from both the
affine-invariant geodesic and the log-Euclidean path
\cite{arsigny2007geometric}; our Fisher-length analysis measures deviation from
the affine-invariant Fisher--Rao geodesic specifically.

\paragraph{Gamma product.}
A product of $n$ Gammas with common shape $\alpha$ and individual rates is, in each
factor, a scale family with $g_{\eta\eta}=\alpha/\eta^2$ ($\eta=\alpha/\beta$). By
Lemma~\ref{lem:scale-reduction}, its inefficiency is the same functional, so
$R^\ast=\Theta(\sqrt{\log d})$ as well. This manifold is Riemann-flat, yet
dimensionally unbounded. Paired with the Hadamard covariance manifold, it proves
that curvature alone does not classify the behavior: the Gamma and covariance
examples share infinite-diameter scale structure despite opposite curvature
signs.

\subsection{Operational meaning of $R$: a step-count corollary}
\label{sec:operational}

The regimes above describe the geometry of $R$; we now give $R$ a concrete algorithmic interpretation. As defined, $R$ is a property of a path's
\emph{shape}. The ratio acquires rigorous operational meaning when each
update's expectation-coordinate chord is capped in the current Fisher norm,
\[
\ell_k:=\|\eta_{k+1}-\eta_k\|_{x_k}\le\varepsilon.
\]
Under this budget, local arc--chord comparison makes the number of steps along
a smooth path asymptotic to its length divided by $\varepsilon$. Let
$N_{\mathrm{geo}}(\varepsilon)$ be the minimum number of Fisher-chord-capped
steps needed to inscribe a minimizing geodesic between $p$ and $q$. On the
compact slab used in the proof,
$N_{\mathrm{geo}}(\varepsilon) = (\delta_{\mathrm{geo}}(p,q)/\varepsilon)(1+O(\varepsilon))$.
Exact forward-KL trust-region or proximal steps remain on the $m$-geodesic, and
the consistent inexact schemes defined in the appendix shadow it to first order;
thus NGD cannot shorten its route by cutting across the Fisher geometry.
Let $N_{\mathrm{NGD}}$ be the number of NGD steps under the per-step budget,
terminating at the first $N$ with $\|\eta_N - \eta^\star\|_{x_N} \le \varepsilon$.
Full statements and proofs: Appendix~\ref{app:opcost}.

\begin{proposition}[One-sided budgeted step-count bound]\label{prop:steps}
Fix $p,q\in M$ and run NGD on the forward-KL objective toward $q$ under the
per-step budget $\ell_k\le\varepsilon$, using the exact trust-region (or
proximal) step, or any $\beta$-consistent budgeted scheme
(Definition~\ref{app:def-consistent}) with $\varepsilon\le\varepsilon_0$,
until $\|\eta_N-\eta^\star\|_{x_N}\le\varepsilon$. Then
\begin{align*}
N_{\mathrm{NGD}}
&\ge\frac{L(\gamma^{(m)})}{\varepsilon}
       \bigl(1-O(\varepsilon)\bigr)\\
&=R(p,q)\,\frac{\delta_{\mathrm{geo}}(p,q)}{\varepsilon}
       \bigl(1-O(\varepsilon)\bigr).
\end{align*}
For fixed distinct $p,q$, this implies
$N_{\mathrm{NGD}}/N_{\mathrm{geo}}(\varepsilon)
\ge R(p,q)(1-O(\varepsilon))$, with explicit constants in
Appendix~\ref{app:opcost}.
\end{proposition}

Thus, asymptotically under the formal Fisher-chord cap, NGD has a one-sided
step-count lower bound of $R(1-o(1))$ relative to the geodesic oracle. Replacing the formal Fisher-chord cap
with a fixed KL cap requires a uniform third-order remainder bound and remains
open. Fixed-rate NGD, which does not obey the chord budget, lies outside the
proposition's scope.
\section{Empirical $R$ Study}
\label{sec:empirics}

The preceding sections established worst-case theoretical bounds; we now examine how $R$ behaves on randomly sampled endpoint pairs. The experiments
ask three questions, one per demonstrated regime:
\begin{enumerate}
\item[(Q1)] \emph{Fisher-efficient regime (I).} Is the $m$-geodesic indeed optimal, i.e.\ is
$R=1$ realized numerically?
\item[(Q2)] \emph{Categorical example within Regime II.} How often, and under
what configurations, is the categorical worst-case bound
$\pi/(2\sqrt2)$ actually approached, and is the bound really independent of dimension?
\item[(Q3)] \emph{Dimensionally unbounded examples (Regime III).} Does the observed inefficiency grow with
dimension (as the theory predicts), how large is it in absolute terms, and how does it compare with the
elementary $\sqrt d$ bracket?
\end{enumerate}

\subsection{Experimental setup}
\label{sec:emp-setup}

For each family we draw $N=2{,}000$ source/target pairs $(p,q)$ per dimension
$d\in\{2,4,8,16\}$. For each pair we record the geodesic length
$\delta_{\mathrm{geo}}(p,q)$, the ratio $R(p,q)$, and a scalar dispersion statistic $H$ of the move.

\paragraph{Sampling of $(p,q)$.}
Pairs are drawn to span a range of both length and spectral shape, so that no
single regime dominates the sampled batch.
\begin{itemize}
\item \emph{Categorical.} $p$ and $q$ are independent draws from a symmetric
Dirichlet distribution whose concentration is drawn per pair on a log scale in
$[0.05,5]$. Small concentration places mass near the simplex boundary
(near-deterministic distributions); large concentration yields interior,
near-uniform distributions.
\item \emph{SPD.} By congruence invariance we fix $\Sigma_p=I_d$; a raw relative
log-spectrum $z\in\mathbb{R}^d$ is drawn from three equally weighted regimes:
near-isotropic ($z=c\mathbf{1}$ plus small noise), log-normal
($z_i\sim\mathcal N(0,\sigma^2)$ with $\sigma$ drawn in $[0.2,2.5]$), and
heavy-tailed (log-uniform magnitudes spanning one to eight decades with random
signs). We then set $\Sigma_q=\exp(\operatorname{diag}z)$, so that the relative
eigenvalues are $e^{z_i}$.
\item \emph{Gamma product.} A sampled Fisher-normalized displacement $w$ is
mapped to raw log expectation ratios $z_i=w_i/\sqrt\alpha$ (equivalently,
opposite log rate ratios) at a fixed common shape $\alpha$ with unit reference
rates. Thus, if $v$ denotes a common
numerical draw and $F$ the shared reduced functional of
Lemma~\ref{lem:scale-reduction}, the plotted protocols use $F(v)$ for SPD but
$F(v/\sqrt\alpha)$ for Gamma. Exact equality is predicted
only when the same \emph{raw} log-ratio vector $z$ is supplied to both models,
not when the same Fisher-normalized displacement is supplied.
\item \emph{Gaussian location} (sanity $R=1$, results omitted below). Random fixed
covariance $\Sigma_0$ and random means $\mu_p,\mu_q$.
\end{itemize}

\paragraph{Geodesic length.}
Each family uses its closed-form Fisher--Rao distance:
categorical, $\delta_{\mathrm{geo}}=2\arccos\sum_i\sqrt{p_iq_i}$;
SPD, $\delta_{\mathrm{geo}}=2^{-1/2}\lVert\log(\Sigma_p^{-1/2}\Sigma_q\Sigma_p^{-1/2})\rVert_F$;
Gamma product, $\delta_{\mathrm{geo}}=\lVert w\rVert_2$ (a product of
one-dimensional scale families);
Gaussian location, the Mahalanobis distance
$\big((\mu_p-\mu_q)^\top\Sigma_0^{-1}(\mu_p-\mu_q)\big)^{1/2}$.
The SPD panels in Figure~\ref{fig:R-vs-length} used the common unnormalized affine-invariant
convention, omitting the factor $2^{-1/2}$ on its horizontal axis; the displayed
SPD distances should therefore be divided by $\sqrt2$. The ratio $R$ is
unchanged because the same factor multiplies both path length and distance.
The $m$-geodesic length is computed by the stabilized log-substitution
integrator of Appendix~\ref{apx:sec:classIII}, and
$R=L(\gamma^{(m)})/\delta_{\mathrm{geo}}$.

\paragraph{Dispersion entropy.}
To quantify how a sampled move is distributed across coordinates, we use a
normalized entropy of the absolute coordinate-difference vector. This is the
raw relative log-spectrum $z$ for SPD and the Fisher-normalized displacement
$w$ for Gamma. Writing either vector generically as $v\in\mathbb{R}^N$,
\[
H(v)=-\frac{1}{\log N}\sum_{i=1}^{N}\hat v_i\log\hat v_i,
\qquad \hat v_i=\frac{|v_i|}{\sum_{j}|v_j|},
\]
and for the categorical family $v_i$ is replaced by $q_i-p_i$ (so $N=k=d+1$). Here $H\in[0,1]$:
$H=0$ when one coordinate dominates and $H=1$ when all coordinates contribute
equally.

\subsection{Results}
\label{sec:emp-results}

Figure~\ref{fig:R-vs-length} plots $R$ against $\delta_{\mathrm{geo}}$, with
color encoding $H$. The Gaussian location regime is omitted: numerically it
returns $R=1$ to machine precision ($|R-1|<10^{-12}$) for every pair and every
$d$, positively answering (Q1) and matching Theorem~\ref{thm:classI}.

\begin{figure*}[t]
\centering
\includegraphics[trim={0 0 0 8cm}, clip, width=\textwidth]{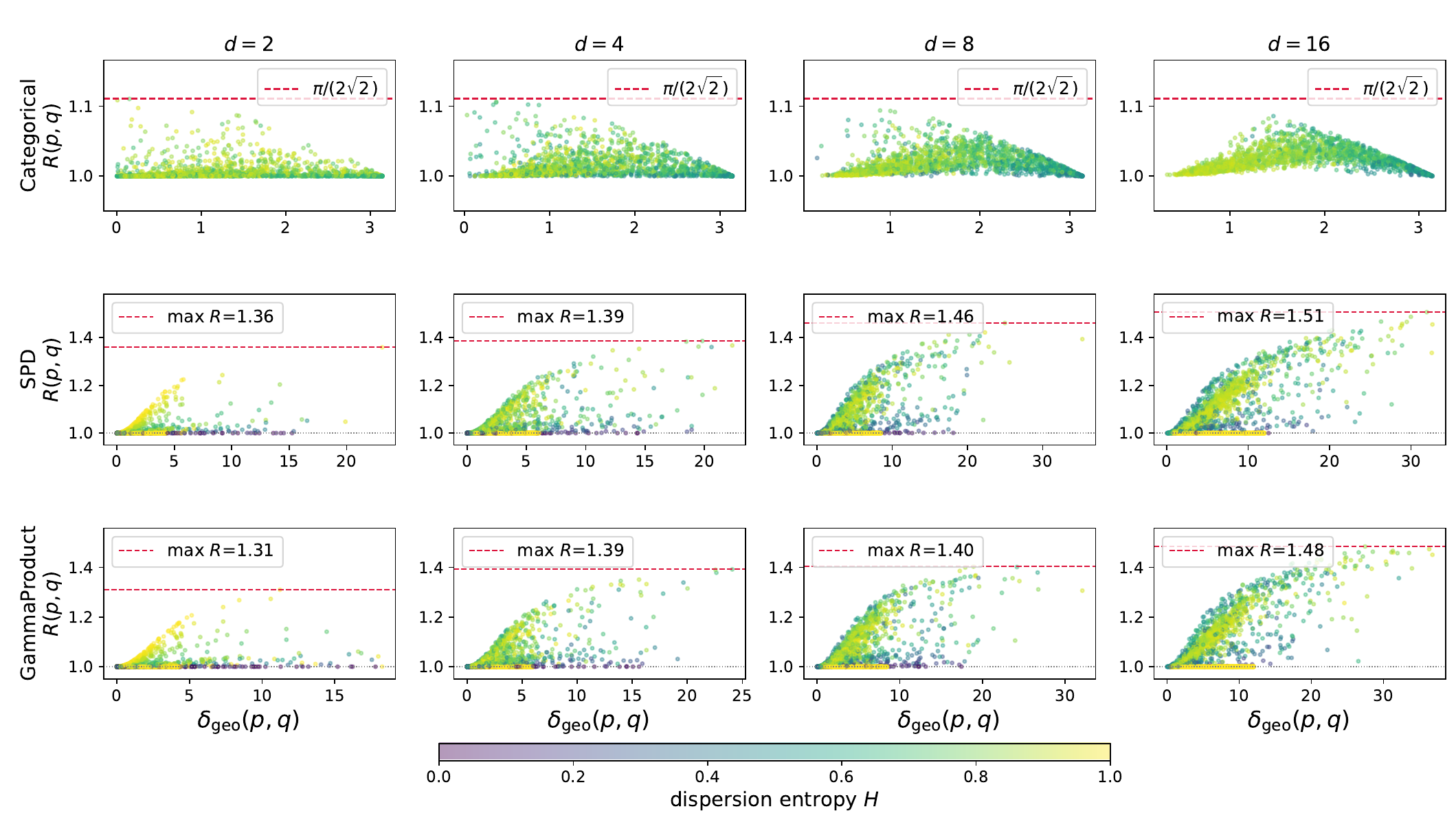}
\caption{Fisher inefficiency $R$ versus Fisher--Rao length
$\delta_{\mathrm{geo}}$, faceted by family and dimension; color encodes
dispersion entropy $H$. The Fisher-efficient example is omitted (verified $R=1$). Categorical panels
mark the theoretical cap $\pi/(2\sqrt2)$; SPD and Gamma panels mark the sampled
maximum. $H$ is computed on a different underlying vector per family (the raw
relative log-spectrum for SPD, the Fisher-normalized displacement for Gamma;
Section~\ref{sec:emp-setup}), so $H$ values should be compared within a family,
not across the SPD and Gamma panels.}
\label{fig:R-vs-length}
\end{figure*}

\paragraph{Categorical (Q2).}
For the categorical family the column index $d$ is the manifold dimension, so a
panel at dimension $d$ has $k=d+1$ outcomes. In every panel $R$ stays strictly
below the $k$-independent cap $\pi/(2\sqrt2)\approx1.1107$
(Theorem~\ref{thm:categorical}), and the observed maxima are close to it and
nearly constant across $d$: $1.110$, $1.106$, $1.094$, and $1.086$ for
$d=2,4,8,16$ (that is, $k=3,5,9,17$). The small decrease with $d$ is consistent
with a sampling artifact rather than a geometric effect: the extremal
configuration changes only two coordinates while the rest are equal in $p$ and
$q$, and such configurations are drawn less often under symmetric Dirichlet
sampling as $k$ grows. Within each
panel, $\delta_{\mathrm{geo}}(p,q)$ and $H$ tend to be inversely related. Thus
Q2 has a mixed answer: the cap $\pi/(2\sqrt2)$ is rarely reached under random
sampling, but low-dimensional short moves and higher-dimensional
intermediate-length moves approach the ${\sim}11\%$ inefficiency bound more often.

\paragraph{SPD and Gamma (Q3).}
The Riemannian-curved SPD family and Riemannian-flat Gamma product exhibit the same
qualitative trend. Their reported maxima need not coincide because, under the
stated normalization, the panels do not use identical raw $z$ vectors. For very short  
$\delta_{\mathrm{geo}}$ the ratio is pinned at $R\approx1$ regardless of
dimension. The upper envelope lifts only at larger $\delta_{\mathrm{geo}}$, and
it rises with $d$: observed maxima grow from $R\approx1.36$ ($d=2$) to
$R\approx1.51$ ($d=16$) for SPD, and from $1.31$ to $1.48$ for Gamma. These
values are below the elementary bracket $\sqrt d$ (which reaches $4$ at
$d=16$), and the growth with $d$ is consistent with the theoretical rate
$\Theta(\sqrt{\log d})$ of Theorem~\ref{thm:class3}. Since $R\ge1$ by
definition, the natural check is an affine fit $R\approx a+b\sqrt{\log_2d}$
rather than a bare multiple of $\sqrt{\log_2d}$; Table~\ref{tab:scale-fit}
shows this fit tracks both families closely, with nearly identical slopes
($b\approx0.15$--$0.16$) despite SPD's Hadamard curvature and Gamma's
flatness. Because these scale-family products have bounded $\kappa_T$ but
infinite diameter, the trend is also evidence that bounded skewness alone
cannot replace the diameter hypothesis in Theorem~\ref{thm:bounded}.

\begin{table}[b]
\centering
\small
\begin{tabular}{@{}lcccc@{}}
\toprule
$d$ & 2 & 4 & 8 & 16 \\
$\sqrt{\log_2 d}$ & $1.00$ & $1.41$ & $1.73$ & $2.00$ \\
\midrule
SPD, observed & $1.36$ & $1.39$ & $1.46$ & $1.51$ \\
SPD, fit $1.19+0.15\sqrt{\log_2 d}$ & $1.35$ & $1.41$ & $1.46$ & $1.50$ \\
\midrule
Gamma, observed & $1.31$ & $1.39$ & $1.40$ & $1.48$ \\
Gamma, fit $1.16+0.16\sqrt{\log_2 d}$ & $1.31$ & $1.38$ & $1.43$ & $1.47$ \\
\bottomrule
\end{tabular}
\caption{Sampled maxima of $R$ against a least-squares affine fit in
$\sqrt{\log_2 d}$. Both families are matched to within $0.03$ at every
dimension by a fit with nearly the same slope, supporting the
$\Theta(\sqrt{\log n})$ rate of Theorem~\ref{thm:class3}.}
\label{tab:scale-fit}
\end{table}

Among the sampled pairs at $d=2$, the larger observed ratios co-occur with both
large $\delta_{\mathrm{geo}}(p,q)$ and two comparable-magnitude coordinates
(high $H$). Sampled moves with a single dominant coordinate (low $H$) were
near-collinear and had $R\approx1$, while sampled high-$H$ pairs at very short
distance ($\delta_{\mathrm{geo}}(p,q)<1$) also remained near $R=1$. These finite samples suggest an interaction
between dispersion and length; they do not establish high $H$ as a necessary
condition.
For $d\ge4$ no single-axis relationship between $H$ and $R$ is
visible, since at higher dimension $H$ alone no longer determines the spectral
configuration that governs $R$.
For both scale families, sampled long, high-dimensional moves can reach
$R\approx1.3$--$1.5$, corresponding to roughly $30$--$50\%$ excess Fisher
length; this pattern is concentrated among pairs with $\delta_{\mathrm{geo}}>20$.

\paragraph{Empirical conclusions.}
The results answer the three questions empirically. The Fisher-efficient
regime attains $R=1$ exactly (Q1). The categorical family is bounded by an
exact, dimension-free constant. Although this bound is tight in the worst
case, it is rarely approached under random sampling (Q2). The sampled scale
products show increasing inefficiency with
dimension, but slowly and at modest absolute values below the $\sqrt d$
bracket (Q3). Across these dually nontrivial families, $R$ is not determined by geodesic length
alone.
The practical message is conditional but concrete. If an optimizer follows the ideal natural-gradient trajectory, then the combination of a
dimensionally unbounded scale family, high dimension, and long geodesic displacement can produce meaningful overhead relative to the shortest Fisher path.

\section{Summary}
\label{sec:summary}

Taken together, the theoretical and empirical results paint a consistent picture. We studied how far natural gradient's idealized trajectory---the mixture
geodesic---strays from the shortest Fisher path, measured by the inefficiency
ratio $R=L(\gamma^{(m)})/\delta_{\mathrm{geo}}\ge1$. A single pointwise
condition, proportionality of the Amari--Chentsov tensor to the metric along the
path, governs local optimality, and demanding it in all directions
characterizes the Fisher-efficient families as exactly those with a quadratic
potential. Away from that boundary, three demonstrated regimes emerge:
(I) Fisher efficient ($R^\ast=1$); (II) uniformly bounded inefficiency, controlled by the dimension-free
bound $R^\ast\le e^{\kappa_TD/2}$ and realized sharply by the categorical
constant $\pi/(2\sqrt2)$; and (III) dimensionally unbounded inefficiency
$R^\ast=\Theta(\sqrt{\log d})$ for scale-family products. Because the Hadamard
covariance manifold and the flat Gamma-rate product both diverge, curvature
alone does not decide the regime---infinite-diameter scale stratification is the
operative mechanism.

Under a per-step Fisher-chord budget, $R$ gives a one-sided
lower bound on natural gradient's step count against a geodesic oracle
(Proposition~\ref{prop:steps}). This is a one-sided statement under the formal
Fisher-chord budget, with a geodesic oracle as comparator.
Within these limits, the demonstrated regimes correspond to (I) no overhead, (II)
uniformly bounded overhead, and (III) dimensionally unbounded overhead. Two
theoretical directions remain open: (1) a
converse deciding whether diameter alone controls the regime, and (2) a version of
the operational bound that holds under a fixed KL cap via a uniform third-order
remainder estimate.
From a practical perspective, these results motivate optimizers that more
closely track Fisher--Rao geodesics where $R$ is large, especially for long,
high-dimensional moves in scale or covariance parameters.


\appendix
\onecolumn

\section{Standing facts about dually flat manifolds}
\label{apx:sec:facts}

We collect the external structural facts used repeatedly. All are standard
\cite{Amari2000MethodsOI,amari2016information,shima2007geometry}; we state them
so that the proofs invoke nothing unstated.

\begin{description}
\item[(F1) Legendre duality.] There is a strictly convex potential $\psi$ on
the natural chart with $\eta=\nabla\psi(\theta)$ and a Legendre-dual potential
$\psi^\ast$ with $\theta=\nabla\psi^\ast(\eta)$. The Fisher metric is the
Hessian $g=\nabla^2\psi$ in the $\theta$-chart, so $\partial\eta/\partial\theta
=\nabla^2\psi=g$.
\item[(F2) Flatness of the dual charts.] $\nabla^{(e)}$ has vanishing
Christoffel symbols in $\theta$; $\nabla^{(m)}$ has vanishing Christoffel
symbols in $\eta$. Hence $e$-geodesics are $\theta$-affine segments and
$m$-geodesics are $\eta$-affine segments:
\[
\theta(\gamma^{(e)}(t))=(1-t)\theta(p)+t\theta(q),\qquad
\eta(\gamma^{(m)}(t))=(1-t)\eta(p)+t\eta(q).
\]
\item[(F3) Duality identity.] For all vector fields $X,Y,Z$,
\[
X\,g(Y,Z)=g\!\left(\nabla^{(m)}_XY,\,Z\right)+g\!\left(Y,\,\nabla^{(e)}_XZ\right).
\]
\item[(F4) Levi--Civita is the average.]
$\nabla^{\mathrm{LC}}=\tfrac12\!\left(\nabla^{(m)}+\nabla^{(e)}\right)$, and its
autoparallel curves are the length-minimizing (Fisher--Rao) geodesics.
\item[(F5) The tensor identity.] Lowering the third index with $g$,
$T(X,Y,Z)=g\!\left(\nabla^{(e)}_XY-\nabla^{(m)}_XY,\,Z\right)$, and $T$ is fully
symmetric. In the $\theta$-chart $T_{ijk}=-\partial_i g_{jk}
=-\partial_i\partial_j\partial_k\psi$.
\end{description}

We also use the elementary vector-integral (triangle) inequality
$\big\|\int_0^1 V(s)\,\mathrm{d} s\big\|\le\int_0^1\|V(s)\|\,\mathrm{d} s$ for any Bochner-integrable
$T_xM$-valued map at a fixed base point, and Cauchy--Schwarz in $(T_xM,g_x)$.


\section{The pointwise criterion and Class I}
\label{apx:sec:classI}

This appendix supplies complete proofs of the pointwise pregeodesy criterion
(Theorem~\ref{thm:pointwise}), the tensor lemma (Lemma~\ref{lem:tensor}), the
Fisher-efficient characterization (Theorem~\ref{thm:classI}), and the
one-dimensional case (Lemma~\ref{lem:dim-one}).

\subsection{Pointwise pregeodesy criterion}

\begin{theorem}[Pointwise pregeodesy criterion (Restated from Theorem~\ref{thm:pointwise})]
\label{apx:thm:pointwise}
Let $p,q\in M$ be distinct and let $\dot\gamma$ be the velocity of their
$m$-geodesic $\gamma$. Then the image of $\gamma$ is a Levi--Civita pregeodesic
if and only if there is a scalar field $\mu(s)$ along $\gamma$ with
\begin{equation}
\label{apx:eq:pointwise}
T(\dot\gamma,\dot\gamma,Z)=\mu\,g(\dot\gamma,Z)\qquad\text{for all vector fields }Z.
\end{equation}
Consequently $R(p,q)=1$ implies \eqref{apx:eq:pointwise}; and \eqref{apx:eq:pointwise}
implies $R(p,q)=1$ whenever the image of $\gamma$ is globally
length-minimizing (e.g.\ if $M$ is Hadamard, or the image lies in a strongly
convex normal neighborhood).
\end{theorem}

\begin{proof}
\emph{Step 1: the covariant acceleration of an $m$-geodesic under
$\nabla^{\mathrm{LC}}$.}
Because $\gamma$ is an $m$-geodesic, (F2) gives
$\nabla^{(m)}_{\dot\gamma}\dot\gamma=0$. Using (F4),
\[
\nabla^{\mathrm{LC}}_{\dot\gamma}\dot\gamma
=\tfrac12\nabla^{(m)}_{\dot\gamma}\dot\gamma
+\tfrac12\nabla^{(e)}_{\dot\gamma}\dot\gamma
=\tfrac12\nabla^{(e)}_{\dot\gamma}\dot\gamma .
\]
By (F5), for any $Z$,
$g\!\left(\nabla^{(e)}_{\dot\gamma}\dot\gamma-\nabla^{(m)}_{\dot\gamma}\dot\gamma,\,Z\right)
=T(\dot\gamma,\dot\gamma,Z)$, and since the $m$-term vanishes,
$g\!\left(\nabla^{(e)}_{\dot\gamma}\dot\gamma,\,Z\right)=T(\dot\gamma,\dot\gamma,Z)$.
Therefore
\begin{equation}
\label{apx:eq:LCacc}
g\!\left(\nabla^{\mathrm{LC}}_{\dot\gamma}\dot\gamma,\,Z\right)
=\tfrac12\,T(\dot\gamma,\dot\gamma,Z)\qquad\text{for all }Z.
\end{equation}

\emph{Step 2: pregeodesy is \eqref{apx:eq:pointwise}.}
The image of $\gamma$ is a Levi--Civita pregeodesic exactly when the
acceleration is proportional to the velocity, i.e.\ there is a scalar
$\tilde\mu(s)$ with $\nabla^{\mathrm{LC}}_{\dot\gamma}\dot\gamma=\tilde\mu\,\dot\gamma$.
Pairing with an arbitrary $Z$ and using \eqref{apx:eq:LCacc},
\[
\nabla^{\mathrm{LC}}_{\dot\gamma}\dot\gamma=\tilde\mu\,\dot\gamma
\iff \tfrac12 T(\dot\gamma,\dot\gamma,Z)=\tilde\mu\,g(\dot\gamma,Z)\ \ \forall Z
\iff T(\dot\gamma,\dot\gamma,Z)=\mu\,g(\dot\gamma,Z)\ \ \forall Z,
\]
where $\mu:=2\tilde\mu$. (The middle equivalence is because a covector is
determined by its pairing with all $Z$, and $g$ is nondegenerate.) This proves
\eqref{apx:eq:pointwise} $\iff$ pregeodesy.

\emph{Step 3: necessity and the conditional converse.}
If $R(p,q)=1$ then $L(\gamma^{(m)})=\delta_{\mathrm{geo}}(p,q)$, so the image of $\gamma$
attains the infimal length between $p$ and $q$ and is therefore a minimizing
geodesic up to reparametrization; in particular it is a Levi--Civita
pregeodesic, so \eqref{apx:eq:pointwise} holds. Conversely, assume
\eqref{apx:eq:pointwise}; by Step~2 the image is a Levi--Civita pregeodesic, i.e.\ a
geodesic up to reparametrization. Under either global hypothesis (Hadamard, or a
strongly convex normal neighborhood) a geodesic segment between its endpoints is
length-minimizing, so its length equals $\delta_{\mathrm{geo}}(p,q)$ and $R(p,q)=1$.
\end{proof}

\subsection{The tensor lemma}

\begin{lemma}[Tensor lemma (Restated from Lemma~\ref{lem:tensor})]
\label{apx:lem:tensor}
Let $A$ be a fully symmetric $(0,3)$-tensor on a Riemannian manifold of
dimension $\ge 2$. If for every tangent vector $v$ there is a scalar
$\lambda(v)$ with
\begin{equation}
\label{apx:eq:tensorhyp}
A(v,v,Z)=\lambda(v)\,g(v,Z)\qquad\text{for all }Z,
\end{equation}
then $A\equiv 0$.
\end{lemma}

\begin{proof}
Work at a fixed point in a $g$-orthonormal frame $e_1,\dots,e_n$, so
$g(e_i,e_j)=\delta_{ij}$ and $A_{ijk}=A(e_i,e_j,e_k)$ is fully symmetric. Then
$g(v,Z)$ is the Euclidean inner product and \eqref{apx:eq:tensorhyp} is a family of
scalar identities.

\emph{Basis vectors.} Put $v=e_i$ in \eqref{apx:eq:tensorhyp} and read the
$e_m$-component: $A_{iim}=\lambda(e_i)\,g(e_i,e_m)=\lambda(e_i)\delta_{im}$.
Hence
\begin{equation}
\label{apx:eq:Aiim}
A_{iim}=0\quad\text{for }m\ne i.
\end{equation}

\emph{Case $\dim\ge 3$.} Fix distinct $i,j$ and an index $m\notin\{i,j\}$. Put
$v=e_i+e_j$. The left side of \eqref{apx:eq:tensorhyp} in the $e_m$-slot is
\[
A(e_i+e_j,e_i+e_j,e_m)=A_{iim}+2A_{ijm}+A_{jjm}
=0+2A_{ijm}+0=2A_{ijm},
\]
using \eqref{apx:eq:Aiim} ($m\ne i$ and $m\ne j$). The right side is
$\lambda(v)\,g(e_i+e_j,e_m)=\lambda(v)(\delta_{im}+\delta_{jm})=0$ since
$m\notin\{i,j\}$. Hence
\begin{equation}
\label{apx:eq:Aijm}
A_{ijm}=0\quad\text{for pairwise distinct }i,j,m.
\end{equation}
Combining \eqref{apx:eq:Aiim} and \eqref{apx:eq:Aijm}: every component with at least two
distinct indices vanishes, leaving only the diagonal entries $A_{iii}$. To kill
these, fix distinct $i,j$ and put $v=e_i+2e_j$ (so $v^i=1,v^j=2$). Reading the
$e_i$-slot,
\[
A(v,v,e_i)=\sum_{a,b}v^av^bA_{abi}
=A_{iii}+4A_{iji}+4A_{jji}=A_{iii},
\]
because $A_{iji}=A_{iij}=0$ and $A_{jji}=0$ by \eqref{apx:eq:Aiim}; the right side is
$\lambda(v)g(v,e_i)=\lambda(v)\,v^i=\lambda(v)$, so $A_{iii}=\lambda(v)$.
Reading the $e_j$-slot,
\[
A(v,v,e_j)=A_{iij}+4A_{ijj}+4A_{jjj}=4A_{jjj},
\]
again by \eqref{apx:eq:Aiim}, while the right side is $\lambda(v)v^j=2\lambda(v)$;
hence $4A_{jjj}=2\lambda(v)$, i.e.\ $\lambda(v)=2A_{jjj}$. Comparing the two
readings, $A_{iii}=2A_{jjj}$. Swapping the roles of $i,j$ (i.e.\ using
$v=2e_i+e_j$) gives $A_{jjj}=2A_{iii}$. Substituting,
$A_{iii}=2A_{jjj}=2(2A_{iii})=4A_{iii}$, so $3A_{iii}=0$ and $A_{iii}=0$ for all
$i$. Thus $A\equiv 0$.

\emph{Case $\dim=2$.} Only indices $1,2$ occur. From \eqref{apx:eq:Aiim} with
$(i,m)=(1,2)$ and $(2,1)$: $A_{112}=0$ and $A_{221}=0$; by symmetry also
$A_{122}=A_{221}=0$ and $A_{121}=A_{112}=0$. Only $A_{111},A_{222}$ can be
nonzero. Put $v=e_1+e_2$. Reading the $e_1$-slot,
\[
A(v,v,e_1)=A_{111}+2A_{121}+A_{221}=A_{111},
\]
with right side $\lambda(v)g(v,e_1)=\lambda(v)$, so $A_{111}=\lambda(v)$.
Reading the $e_2$-slot,
$A(v,v,e_2)=A_{112}+2A_{122}+A_{222}=A_{222}=\lambda(v)$. Hence
$A_{111}=A_{222}$. Now put $v=e_1+2e_2$. The $e_1$-slot gives
$A(v,v,e_1)=A_{111}+4A_{121}+4A_{221}=A_{111}=\lambda(v)$; the $e_2$-slot gives
$A(v,v,e_2)=A_{112}+4A_{122}+4A_{222}=4A_{222}=\lambda(v)g(v,e_2)=2\lambda(v)$,
so $4A_{222}=2A_{111}$, i.e.\ $A_{111}=2A_{222}$. With $A_{111}=A_{222}$ this
forces $A_{222}=2A_{222}$, hence $A_{222}=0$ and then $A_{111}=0$. So $A\equiv0$.
\end{proof}

\subsection{Dimension one}

\begin{lemma}[Dimension one (Restated from Lemma~\ref{lem:dim-one})]
\label{apx:lem:dim-one}
Let $M$ be a one-dimensional dually flat family whose global expectation image
$\eta(M)$ is an interval. Then $R\equiv 1$, regardless of $T$.
\end{lemma}

\begin{proof}
In one dimension write the metric as $g=g(\eta)\,\,\mathrm{d}\eta^2$ in the expectation
chart, with $g(\eta)>0$. Let $p,q$ correspond to $\eta_p<\eta_q$ (the other
order is symmetric). The Fisher length of \emph{any} $C^1$ curve $c$ from $p$ to
$q$, parametrized so that $\eta$ runs over $[\eta_p,\eta_q]$ (possibly
non-monotonically), satisfies
\[
L(c)=\int \sqrt{g(\eta(c))}\,|\dot\eta_c|\,\,\mathrm{d} t
\ge\int_{\eta_p}^{\eta_q}\sqrt{g(\eta)}\,\,\mathrm{d}\eta,
\]
because the total variation of $\eta$ along $c$ is at least the length
$\eta_q-\eta_p$ of the interval it must cover, and the integrand is a positive
function of $\eta$ alone (any back-and-forth only adds length). The $m$-geodesic
traverses $[\eta_p,\eta_q]$ monotonically once, so it attains the lower bound
$\int_{\eta_p}^{\eta_q}\sqrt{g(\eta)}\,\,\mathrm{d}\eta=\delta_{\mathrm{geo}}(p,q)$. Hence
$L(\gamma^{(m)})=\delta_{\mathrm{geo}}(p,q)$ and $R\equiv1$. The tensor lemma requires
$\dim\ge2$ and does not apply, so $T$ may be nonzero.
\end{proof}

\subsection{Fisher-efficient characterization}

\begin{theorem}[Fisher-efficient characterization (Restated from Theorem~\ref{thm:classI})]
\label{apx:thm:classI}
Assume $\eta(M)$ is convex. If $\dim M=1$, then $R^{\ast}(M)=1$. If $\dim M\ge2$,
then $R^{\ast}(M)=1$ if and only if $\psi$ is quadratic (equivalently $\eta$ and
$\theta$ are affinely related, equivalently $T\equiv0$).
\end{theorem}

\begin{proof}
\emph{Dimension one.} Immediate from Lemma~\ref{apx:lem:dim-one}.

\emph{``If'' ($\psi$ quadratic $\Rightarrow R^{\ast}=1$).} If $\psi$ is quadratic
then $\eta=\nabla\psi(\theta)$ is an affine function of $\theta$ (F1), so
$g=\nabla^2\psi$ is constant. By (F5), $T=\nabla^3\psi\equiv0$, hence
$\nabla^{(e)}=\nabla^{(m)}$, and by (F4) both equal $\nabla^{\mathrm{LC}}$. Thus every
$m$-geodesic is a Levi--Civita geodesic. The convex-domain hypothesis guarantees
the $\eta$-affine segment between any $p,q$ stays in $M$; for a constant metric
this straight segment is the unique minimizer, so $R\equiv1$ and $R^{\ast}=1$.

\emph{``Only if'' ($\dim\ge2$, $R^{\ast}=1\Rightarrow\psi$ quadratic).} If
$R^{\ast}=1$ then $R(p,q)=1$ for every pair, so by Theorem~\ref{apx:thm:pointwise} the
pointwise identity \eqref{apx:eq:pointwise} holds for every $m$-geodesic:
$T(\dot\gamma,\dot\gamma,Z)=\mu\,g(\dot\gamma,Z)$. Fix $p$ and let $q\to p$; the
initial velocities $\dot\gamma(0)$ range over all of $T_pM$ (every direction is
the velocity of some $m$-geodesic, since $\eta$-affine segments issue in every
direction). Hence for every $v\in T_pM$ there is a scalar $\lambda(v)$ with
$T(v,v,Z)=\lambda(v)g(v,Z)$ for all $Z$. Since $\dim\ge2$,
Lemma~\ref{apx:lem:tensor} gives $T\equiv0$ at $p$; as $p$ was arbitrary, $T\equiv0$
on $M$. By (F5) $\nabla^3\psi\equiv0$, so $\psi$ is quadratic (equivalently
$g=\nabla^2\psi$ constant and $\eta,\theta$ affinely related).
\end{proof}

\begin{remark}[worked instances]
\label{apx:rem:classI-examples}
The Gaussian location family $N(\mu,\Sigma_0)$ at fixed $\Sigma_0$ has
$\theta=\Sigma_0^{-1}\mu$ and $\eta=\mu$, affinely related, so $\psi$ is
quadratic and $R\equiv1$ (the discussion following Theorem~\ref{thm:classI};
\cite{skovgaard1984riemannian}). Every one-parameter exponential family is
Fisher efficient by Lemma~\ref{apx:lem:dim-one}, even when $T\ne0$. In dimension
$\ge2$ a non-quadratic potential forces $R>1$ somewhere.
\end{remark}


\section{Class II: the skewness--diameter criterion}
\label{apx:sec:classII}\label{app:skewness}

This appendix supplies a complete proof of the skewness--diameter criterion
(Theorem~\ref{thm:bounded}) and its two supporting results,
Lemmas~\ref{apx:lem:polar} and~\ref{apx:lem:ruler} below. The argument uses \emph{no comparison
geometry}: no Rauch comparison, no Jacobi fields, no injectivity radius. Its
engine is the distortion of a fixed $m$-affine ruler along curves.

Recall the $g$-normalized skewness
\[
\kappa_T(M)=\sup_{x\in M}\ \sup_{v\ne0}\frac{|T_x(v,v,v)|}{g_x(v,v)^{3/2}}
=\sup_{x\in M}\ \sup_{\|v\|_x=1}|T_x(v,v,v)|,
\]
a dimension-free scalar invariant, and the Fisher--Rao diameter
$D(M)=\sup_{p,q}\delta_{\mathrm{geo}}(p,q)$.

\subsection{Polarization}

\begin{lemma}[Polarization]
\label{apx:lem:polar}\label{app:lem-polar}
For all $x\in M$ and $w,u\in T_xM$,
\[
|T_x(w,u,u)|\le \kappa_T\,\|w\|_x\,\|u\|_x^2 .
\]
\end{lemma}

\begin{proof}
$T_x$ is a symmetric trilinear form on the Hilbert space $(T_xM,g_x)$. We use
Banach's theorem on symmetric multilinear forms \cite{banach1938homogene}: for a
bounded symmetric $k$-linear form $B$ on a real Hilbert space, the operator norm
equals the supremum of the associated diagonal form,
\[
\sup_{\|u_1\|,\dots,\|u_k\|\le1}|B(u_1,\dots,u_k)|
=\sup_{\|u\|\le1}|B(u,\dots,u)|.
\]
Applied to $B=T_x$ ($k=3$), the right-hand side is $\kappa_T$ by definition, so
\[
\sup_{\|u_1\|,\|u_2\|,\|u_3\|\le1}|T_x(u_1,u_2,u_3)|=\kappa_T .
\]
By trilinearity and homogeneity, for arbitrary $w,u_2,u_3$,
$|T_x(w,u_2,u_3)|\le\kappa_T\,\|w\|_x\|u_2\|_x\|u_3\|_x$; specializing
$u_2=u_3=u$ gives the claim. The constant is $1$, independent of $\dim M$.
\end{proof}

\begin{remark}[citation-free substitute]
\label{apx:rem:polar-elem}
One can avoid Banach's theorem at the cost of a worse constant. With
$P(z):=T_x(z,z,z)$, full symmetry gives the cubic expansion
$P(a+b)=P(a)+3T_x(a,a,b)+3T_x(a,b,b)+P(b)$. Applying it to
$P(u+\tfrac{w}{2})$ and $P(\tfrac{w}{2}-u)=-P(u-\tfrac{w}{2})$ and adding, the
odd-in-$u$ terms cancel and the even ones double, yielding the symmetric-cubic
identity
\[
P\!\left(u+\tfrac{w}{2}\right)+P\!\left(\tfrac{w}{2}-u\right)
=3\,T_x(u,u,w)+\tfrac14P(w).
\]
Solving for $T_x(u,u,w)$ and bounding each $P(\cdot)$ by $\kappa_T\|\cdot\|_x^3$ shows
$|T_x(u,u,w)|\le c_0\,\kappa_T\|u\|_x^2\|w\|_x$ for an absolute constant $c_0$ (the
paper's sharper accounting gives $c_0=\tfrac54$). Using this in place of
Lemma~\ref{apx:lem:polar} only enlarges the exponent in Theorem~\ref{apx:thm:bounded}
from $\tfrac12\kappa_T D$ to $\tfrac{c_0}{2}\kappa_T D$; the finiteness conclusion is
unchanged.
\end{remark}

\subsection{Ruler distortion}

\begin{lemma}[Ruler distortion]
\label{apx:lem:ruler}\label{app:lem-ruler}
Fix a vector $a$ in the $\eta$-coordinate vector space and let $V_a$ be the
vector field with constant $\eta$-components $a$; write $\|a\|_x=g_x(V_a,V_a)^{1/2}$.
Then for all $x,y\in M$,
\[
e^{-\kappa_T\,\delta_{\mathrm{geo}}(x,y)/2}\,\|a\|_x\ \le\ \|a\|_y\ \le\ e^{\kappa_T\,\delta_{\mathrm{geo}}(x,y)/2}\,\|a\|_x .
\]
\end{lemma}

\begin{proof}
Since $V_a$ has constant $\eta$-components and $\nabla^{(m)}$ is flat in $\eta$
(F2), $\nabla^{(m)}V_a=0$. Let $c$ be any piecewise-$C^1$ curve. By the duality
identity (F3) with $X=\dot c$, $Y=Z=V_a$,
\[
\frac{\,\mathrm{d}}{\,\mathrm{d} t}g(V_a,V_a)
=g\!\left(\nabla^{(m)}_{\dot c}V_a,\,V_a\right)+g\!\left(V_a,\,\nabla^{(e)}_{\dot c}V_a\right)
=g\!\left(V_a,\,\nabla^{(e)}_{\dot c}V_a\right).
\]
Now $\nabla^{(e)}_{\dot c}V_a-\nabla^{(m)}_{\dot c}V_a=\nabla^{(e)}_{\dot c}V_a$
(the $m$-term vanishes), so by (F5), $g(V_a,\nabla^{(e)}_{\dot c}V_a)
=T(\dot c,V_a,V_a)$. Therefore
\begin{equation}
\label{apx:eq:rulerspeed}
\frac{\,\mathrm{d}}{\,\mathrm{d} t}g(V_a,V_a)=T(\dot c,V_a,V_a).
\end{equation}
(When $c=\gamma^{(m)}$ and $a$ is its own $\eta$-velocity, this is the classical
speed identity; here $a$ is an \emph{arbitrary fixed} ruler.) Consequently,
where $\|a\|_{c(t)}>0$,
\[
\left|\frac{\,\mathrm{d}}{\,\mathrm{d} t}\log\|a\|_{c(t)}\right|
=\frac{1}{2}\,\frac{|T(\dot c,V_a,V_a)|}{g(V_a,V_a)}
\le\frac{\kappa_T}{2}\,\|\dot c\|_{c(t)},
\]
the inequality by Lemma~\ref{apx:lem:polar} with $w=\dot c$, $u=V_a$. Integrating
over $c$,
\[
\left|\log\frac{\|a\|_y}{\|a\|_x}\right|\le\frac{\kappa_T}{2}\,L(c).
\]
Taking the infimum over piecewise-$C^1$ curves $c$ from $x$ to $y$ replaces
$L(c)$ by $\delta_{\mathrm{geo}}(x,y)$; exponentiating gives the two-sided bound. No
completeness or attained minimizer is needed.
\end{proof}

\begin{remark}[the Riccati trap that is avoided]
\label{apx:rem:riccati}
The naive route---Gr\"onwall on the $m$-geodesic's own speed
$f(s)=\|\dot\gamma\|_{\gamma(s)}$---uses \eqref{apx:eq:rulerspeed} with $a=\dot\gamma$
and yields only the self-referential inequality $|f'|\le\tfrac{\kappa_T}{2}f^2$, a
Riccati-type bound that can blow up in finite time. Lemma~\ref{apx:lem:ruler}
integrates a \emph{fixed} ruler instead, whose logarithmic derivative is bounded
\emph{linearly} in the curve's speed; the diameter then supplies global control.
\end{remark}

\subsection{The bounded-inefficiency theorem}

\begin{theorem}[Bounded inefficiency (Restated from Theorem~\ref{thm:bounded})]
\label{apx:thm:bounded}
Let $\dim M\ge2$. If $\kappa_T(M)<\infty$ and $D(M)<\infty$, then
\[
R^{\ast}(M)\le \exp\!\Big(\frac{\kappa_T(M)\,D(M)}{2}\Big),
\]
independently of $\dim M$. If moreover $\eta(M)$ is convex and $\psi$ is
non-quadratic, then $1<R^{\ast}(M)$.
\end{theorem}

\begin{proof}
\emph{Lower bound.} If $\eta(M)$ is convex and $\psi$ is non-quadratic then, by
the ``only if'' direction of Theorem~\ref{apx:thm:classI}, $R^{\ast}>1$.

\emph{Upper bound.} Fix $p\ne q$, put $h=\eta(q)-\eta(p)$ (a fixed vector in the
$\eta$-coordinate space), and let $\gamma$ be the $m$-geodesic, so
$\eta(\gamma(s))=\eta(p)+sh$ and $\dot\gamma$ has constant $\eta$-components
$h$. The map $s\mapsto\|h\|_{\gamma(s)}$ is continuous on the compact $[0,1]$;
choose $s_\ast$ maximizing it and set $z=\gamma(s_\ast)$, $r=\|h\|_z$.

\emph{Numerator.} The $m$-geodesic's length is
\[
L(\gamma^{(m)})=\int_0^1\|\dot\gamma(s)\|_{\gamma(s)}\,\mathrm{d} s
=\int_0^1\|h\|_{\gamma(s)}\,\mathrm{d} s\le r,
\]
since $\|h\|_{\gamma(s)}\le\|h\|_z=r$ for all $s$.

\emph{Denominator.} Let $c$ be any piecewise-$C^1$ curve from $p$ to $q$, with
$\eta$-velocity $\dot\eta_c(t)$. Apply Lemma~\ref{apx:lem:ruler} to the fixed ruler
$a=\dot\eta_c(t)$ (at each $t$) between the points $z$ and $c(t)$; using
$\delta_{\mathrm{geo}}(z,c(t))\le D$,
\[
\|\dot\eta_c(t)\|_{c(t)}\ \ge\ e^{-\kappa_T D/2}\,\|\dot\eta_c(t)\|_z .
\]
Integrating, and then applying the vector-integral (triangle) inequality in the
\emph{single frozen} inner product $\|\cdot\|_z$,
\[
L(c)=\int_0^1\|\dot\eta_c(t)\|_{c(t)}\,\mathrm{d} t
\ \ge\ e^{-\kappa_T D/2}\int_0^1\|\dot\eta_c(t)\|_z\,\mathrm{d} t
\ \ge\ e^{-\kappa_T D/2}\Big\|\int_0^1\dot\eta_c(t)\,\mathrm{d} t\Big\|_z .
\]
But $\int_0^1\dot\eta_c(t)\,\mathrm{d} t=\eta(q)-\eta(p)=h$, so
$L(c)\ge e^{-\kappa_T D/2}\|h\|_z=e^{-\kappa_T D/2}r$. Taking the infimum over $c$,
\[
\delta_{\mathrm{geo}}(p,q)\ge e^{-\kappa_T D/2}\,r.
\]

\emph{Combine.} $R(p,q)=L(\gamma^{(m)})/\delta_{\mathrm{geo}}(p,q)\le r/\big(e^{-\kappa_T D/2}r\big)
=e^{\kappa_T D/2}$. Taking the supremum over pairs gives
$R^{\ast}(M)\le e^{\kappa_T D/2}$. Every constant is $1$ or $\tfrac12$, so there is no
dependence on $\dim M$.
\end{proof}

\begin{remark}[self-concordance form of the hypothesis]
\label{apx:rem:selfconcordance}
In the natural chart, $T_x(v,v,v)=-D^3\psi_x[v,v,v]$ and $g_x(v,v)=D^2\psi_x[v,v]$,
so $\kappa_T(M)<\infty$ is exactly the self-concordance inequality of
Nesterov--Nemirovskii \cite{nesterov1994interior},
$|D^3\psi_x[v,v,v]|\le\kappa\,(D^2\psi_x[v,v])^{3/2}$, with $\kappa=\kappa_T$. The
theorem says: self-concordance plus finite Fisher--Rao diameter forces bounded
inefficiency uniformly in dimension.
\end{remark}

\begin{remark}[Consistency checks]
\label{apx:rem:consistency}\label{app:rem-skew-consistency}
\emph{(i) Fisher-efficient limit.} If $\kappa_T=0$ then, by Lemma~\ref{apx:lem:polar},
$T\equiv0$; by \eqref{apx:eq:rulerspeed} every $\eta$-constant ruler has constant
norm, so $g$ is constant in $\eta$-coordinates and every competing curve has
length $\ge\|h\|$, giving $R\equiv1$. The bound reads $e^0=1$, and the case
$\kappa_T=0,\,D=\infty$ needs no interpretation of $0\cdot\infty$.

\emph{(ii) Tensorization.} For a product $M=\prod_iM_i$ with $v=(v_i)$ and
$a_i=\|v_i\|_{g_i}$, full symmetry and vanishing of cross terms give
$T(v,v,v)=\sum_iT_i(v_i,v_i,v_i)$, hence
\[
|T(v,v,v)|\le\Big(\max_i\kappa_{T,i}\Big)\sum_ia_i^3
\le\Big(\max_i\kappa_{T,i}\Big)\Big(\sum_ia_i^2\Big)^{3/2},
\]
using $\sum a_i^3\le(\sum a_i^2)^{3/2}$ (monotonicity of $\ell^p$ norms), with
equality approached on single-factor vectors. Thus
$\kappa_T(\prod_iM_i)=\max_i\kappa_T(M_i)$: the constant does not count factors. A product
containing an infinite-diameter factor has $D=\infty$, which is exactly how the
scale products of Section~\ref{apx:sec:classIII} escape this theorem.

\emph{(iii) Categorical compacta.} On
$\Delta_{k-1}^\varepsilon=\{p:\min_ip_i\ge\varepsilon\}$, with $a_i=v_i/\sqrt{p_i}$,
$|T_p(v,v,v)|=\big|\sum_iv_i^3/p_i^2\big|\le\varepsilon^{-1/2}\sum_i(v_i^2/p_i)^{3/2}
\cdot(\cdots)$; more simply $|v_i^3/p_i^2|=|a_i|^3\sqrt{p_i}\cdot p_i^{-1/2}\cdots$,
and the clean bound is $|T_p(v,v,v)|\le\varepsilon^{-1/2}(\sum_ia_i^2)^{3/2}$, so
$\kappa_T\le\varepsilon^{-1/2}$; and $\Delta_{k-1}^\varepsilon$ is geodesically convex
with $D\le\pi$. Both are dimension-free, so
$R^{\ast}\le e^{\pi/(2\sqrt\varepsilon)}$---finite, but far above the sharp
$\pi/(2\sqrt2)$ of Section~\ref{apx:sec:categorical}. On the \emph{full} simplex
$\kappa_T=\infty$ (already Bernoulli: $|1-2\mu|/\sqrt{\mu(1-\mu)}\to\infty$ at the
boundary), so this theorem does \emph{not} give the categorical constant; that
requires the independent argument of Section~\ref{apx:sec:categorical}.
\end{remark}


\section{The categorical constant $\pi/(2\sqrt2)$}
\label{apx:sec:categorical}\label{app:simplex}

This appendix proves the categorical constant (Theorem~\ref{thm:categorical})
in full: the exact value
$R^{\ast}(\Delta_{k-1})=\pi/(2\sqrt2)$ for $k\ge3$ and $R^{\ast}(\Delta_1)=1$. The
lower bound is an explicit two-outcome computation; the upper bound is the
exact arcsine-weight argument developed below: one integral
identity, one Cauchy--Schwarz step, and $\sin x\le x$. No comparison geometry is
used, and the argument does not invoke Theorem~\ref{apx:thm:bounded}.

\paragraph{Setup and the isometry.}
Under $p\mapsto z=2\sqrt p$ (coordinatewise), the interior of $\Delta_{k-1}$ is
isometric to the strictly positive orthant of the radius-$2$ sphere
$S^{k-1}_2\subset\mathbb{R}^k$, because
$\sum_iz_i^2=4\sum_ip_i=4$ and the Fisher metric pulls back to the round metric;
the Fisher--Rao distance is the great-circle arc
\[
\delta_{\mathrm{geo}}(p,q)=2\arccos\!\Big(\sum_i\sqrt{p_iq_i}\Big)
\]
(standard; \cite{miyamoto2024closedform,nielsen2023simple}). Write
$\alpha_i=\sqrt{p_i}$, $\beta_i=\sqrt{q_i}$, and along the $m$-geodesic
$p(s)=(1-s)p+sq$ set
\[
u_i(s)=\sqrt{(1-s)\alpha_i^2+s\beta_i^2}=\tfrac12 z_i(s),\qquad \|u(s)\|_2=1
\]
since $\sum_i u_i(s)^2 = (1-s)\sum_i\alpha_i^2 + s\sum_i\beta_i^2 = (1-s)\cdot4 + s\cdot4 = 4$, normalized by the factor $\tfrac12$.
Then $L(p,q)=2\int_0^1\|\dot u(s)\|_2\,\mathrm{d} s$, and each coordinate's squared speed
is the \emph{reciprocal-affine} function
\begin{equation}
\label{apx:eq:recipaffine}
\dot u_i(s)^2=\frac{(\beta_i^2-\alpha_i^2)^2}{4\big(\alpha_i^2+(\beta_i^2-\alpha_i^2)s\big)} .
\end{equation}
Indeed $u_i^2=(1-s)\alpha_i^2+s\beta_i^2=\alpha_i^2+(\beta_i^2-\alpha_i^2)s$ is
affine in $s$, so $\dot u_i=\tfrac{\,\mathrm{d}}{\,\mathrm{d} s}\sqrt{u_i^2}
=\frac{\beta_i^2-\alpha_i^2}{2u_i}$, whose square is \eqref{apx:eq:recipaffine}.

\subsection{The scalar integral identities}

\begin{lemma}[Scalar integral identity]
\label{apx:lem:it}\label{app:lem-it}
For every $t>0$,
\[
I(t):=\int_0^1\frac{\sqrt{s(1-s)}}{s+t}\,\mathrm{d} s
=\pi\Big(\tfrac12+t-\sqrt{t(1+t)}\Big),
\]
and $I(0^+)=\pi/2$.
\end{lemma}

\begin{proof}
Substitute $s=\sin^2\theta$, $\theta\in[0,\pi/2]$, so
$\,\mathrm{d} s=2\sin\theta\cos\theta\,\,\mathrm{d}\theta$ and
$\sqrt{s(1-s)}=\sin\theta\cos\theta$. Then
\[
I(t)=\int_0^{\pi/2}\frac{\sin\theta\cos\theta}{\sin^2\theta+t}
\cdot 2\sin\theta\cos\theta\,\,\mathrm{d}\theta
=2\int_0^{\pi/2}\frac{\sin^2\theta\cos^2\theta}{\sin^2\theta+t}\,\mathrm{d}\theta .
\]
Introduce $J(t):=\int_0^{\pi/2}\frac{\,\mathrm{d}\theta}{\sin^2\theta+t}$. Substituting
$x=\tan\theta$ ($\,\mathrm{d}\theta=\,\mathrm{d} x/(1+x^2)$, $\sin^2\theta=x^2/(1+x^2)$),
\[
J(t)=\int_0^\infty\frac{1}{\frac{x^2}{1+x^2}+t}\cdot\frac{\,\mathrm{d} x}{1+x^2}
=\int_0^\infty\frac{\,\mathrm{d} x}{x^2+t(1+x^2)}
=\int_0^\infty\frac{\,\mathrm{d} x}{(1+t)x^2+t}
=\frac{\pi}{2\sqrt{t(1+t)}},
\]
using the elementary integral $\int_0^\infty\frac{\,\mathrm{d} x}{ax^2+b}=\frac{\pi}{2\sqrt{ab}}$
($a=1+t$, $b=t$). Now decompose the numerator. Since
$\sin^2\theta\cos^2\theta=(\sin^2\theta+t)\cos^2\theta-t\cos^2\theta$ and
$\cos^2\theta=(1+t)-(\sin^2\theta+t)$,
\[
\frac{\sin^2\theta\cos^2\theta}{\sin^2\theta+t}
=\cos^2\theta-\frac{t\cos^2\theta}{\sin^2\theta+t}
=\cos^2\theta-\frac{t(1+t)}{\sin^2\theta+t}+t .
\]
Integrate over $[0,\pi/2]$, using $\int_0^{\pi/2}\cos^2\theta\,\,\mathrm{d}\theta=\pi/4$
and $\int_0^{\pi/2}\,\mathrm{d}\theta=\pi/2$:
\[
\tfrac12 I(t)=\frac{\pi}{4}-t(1+t)J(t)+t\cdot\frac{\pi}{2}
=\frac\pi4-t(1+t)\frac{\pi}{2\sqrt{t(1+t)}}+\frac{\pi t}{2}
=\frac\pi4+\frac{\pi t}{2}-\frac{\pi}{2}\sqrt{t(1+t)} .
\]
Doubling gives $I(t)=\frac\pi2+\pi t-\pi\sqrt{t(1+t)}
=\pi\big(\tfrac12+t-\sqrt{t(1+t)}\big)$. As $t\downarrow0$, $\sqrt{t(1+t)}\to0$,
so $I(0^+)=\pi/2$.
\end{proof}

\begin{lemma}[Master identity]
\label{apx:lem:master}\label{app:lem-master}
Let $A,B\ge0$ and $u(s)=\sqrt{(1-s)A^2+sB^2}$ on $[0,1]$. Then
\[
\int_0^1\dot u(s)^2\sqrt{s(1-s)}\,\mathrm{d} s=\frac{\pi}{8}\,(B-A)^2 .
\]
\end{lemma}

\begin{proof}
If $A=B$ both sides vanish. The reflection $s\mapsto1-s$ preserves the weight
$\sqrt{s(1-s)}$ and swaps $A\leftrightarrow B$, so we may assume $B>A\ge0$; set
$D=B^2-A^2>0$. By \eqref{apx:eq:recipaffine} (with $A,B$ in place of
$\alpha_i,\beta_i$),
\[
\dot u(s)^2=\frac{D^2}{4(A^2+Ds)}=\frac{D}{4}\cdot\frac{1}{s+t},\qquad t=\frac{A^2}{D}.
\]
Hence, by Lemma~\ref{apx:lem:it},
\[
\int_0^1\dot u^2\sqrt{s(1-s)}\,\mathrm{d} s=\frac{D}{4}I(t)
=\frac{\pi D}{4}\Big(\tfrac12+\tfrac{A^2}{D}-\sqrt{\tfrac{A^2}{D}\cdot\tfrac{B^2}{D}}\Big),
\]
using $1+t=1+A^2/D=(D+A^2)/D=B^2/D$ and hence $\sqrt{t(1+t)}=AB/D$. Simplify the
bracket:
\[
\frac\pi4\Big(\frac D2+A^2-AB\Big)
=\frac\pi4\Big(\frac{B^2-A^2}{2}+A^2-AB\Big)
=\frac\pi4\cdot\frac{B^2+A^2-2AB}{2}
=\frac\pi8(B-A)^2 .
\]
The case $A=0$ is the limit $t\to0$, where $I(0^+)=\pi/2$ gives
$\frac{D}{4}\cdot\frac\pi2=\frac{B^2}{4}\cdot\frac\pi2=\frac\pi8B^2$, consistent.
\end{proof}

The offset $t$ cancels identically: against the arcsine weight, each
coordinate's weighted energy depends on its endpoints only through the squared
chord $(B-A)^2$.

\subsection{The chord bound and the sharp constant}

\begin{theorem}[Chord bound]
\label{apx:thm:chord}\label{app:thm-chord}
For all $p,q\in\Delta_{k-1}$, with $z(p)=2\sqrt p$,
\[
L(p,q)\le\frac{\pi}{2\sqrt2}\,\big\|z(q)-z(p)\big\|_2 .
\]
Equality holds if and only if every coordinate is frozen ($p_i=q_i$), a pure
entry ($p_i=0$), or a pure exit ($q_i=0$).
\end{theorem}

\begin{proof}
Since $z=2\sqrt p$ gives $\|z(q)-z(p)\|_2=2\|\beta-\alpha\|_2$ and
$L(p,q)=2\int_0^1\|\dot u\|_2\,\mathrm{d} s$, it suffices to prove
$\int_0^1\|\dot u\|_2\,\mathrm{d} s\le\tfrac{\pi}{2\sqrt2}\|\beta-\alpha\|_2$. By
Cauchy--Schwarz against the arcsine weight $\sqrt{s(1-s)}$ (writing
$\|\dot u\|_2=\|\dot u\|_2\sqrt{\sqrt{s(1-s)}}\cdot(\sqrt{s(1-s)})^{-1/2}$),
\[
\int_0^1\|\dot u\|_2\,\mathrm{d} s
\le\Big(\int_0^1\|\dot u\|_2^2\sqrt{s(1-s)}\,\mathrm{d} s\Big)^{1/2}
\Big(\int_0^1\frac{\,\mathrm{d} s}{\sqrt{s(1-s)}}\Big)^{1/2}.
\]
The second factor is $\big(\int_0^1 (s(1-s))^{-1/2}\,\mathrm{d} s\big)^{1/2}
=\big(B(\tfrac12,\tfrac12)\big)^{1/2}=\sqrt\pi$ (Beta function; equivalently the
substitution $s=\sin^2\theta$ gives $\int_0^\pi 1\,\,\mathrm{d}\theta$ over
$[0,\pi/2]$ doubled $=\pi$). For the first factor, sum Lemma~\ref{apx:lem:master}
over coordinates:
\[
\int_0^1\|\dot u\|_2^2\sqrt{s(1-s)}\,\mathrm{d} s
=\sum_i\int_0^1\dot u_i^2\sqrt{s(1-s)}\,\mathrm{d} s
=\sum_i\frac\pi8(\beta_i-\alpha_i)^2=\frac\pi8\|\beta-\alpha\|_2^2 .
\]
Therefore
\[
\int_0^1\|\dot u\|_2\,\mathrm{d} s\le\sqrt\pi\cdot\sqrt{\tfrac\pi8}\,\|\beta-\alpha\|_2
=\frac{\pi}{\sqrt8}\|\beta-\alpha\|_2=\frac{\pi}{2\sqrt2}\|\beta-\alpha\|_2 .
\]

\emph{Equality.} Cauchy--Schwarz is tight iff the two integrand factors are
proportional a.e., i.e.\ $\|\dot u(s)\|_2^2\cdot\sqrt{s(1-s)}
=\lambda\cdot(s(1-s))^{-1/2}$ a.e.\ for some constant, equivalently
$\|\dot u(s)\|_2^2=c/(s(1-s))$ a.e.\ for some $c\ge0$. By
\eqref{apx:eq:recipaffine}, $\|\dot u\|_2^2$ is a sum of terms with a simple pole
each at $s=-t_i\le0$ (an increasing coordinate, pole outside $[0,1]$ on the
left) or at $s=1+\tilde t_i\ge1$ (a decreasing coordinate). Matching the pole
set $\{0,1\}$ of $c/(s(1-s))$ forces every moving coordinate to have its pole
exactly at $0$ or $1$: pole at $0$ means $t_i=A_i^2/D_i=0$, i.e.\ $\alpha_i=0$
($p_i=0$, a pure entry); pole at $1$ means the reflected offset vanishes, i.e.\
$\beta_i=0$ ($q_i=0$, a pure exit). Conversely, for such a configuration
\[
\|\dot u\|_2^2=\frac{\sum_{i\in E}q_i}{4s}+\frac{\sum_{i\in X}p_i}{4(1-s)},
\]
where $E,X$ index pure entries and exits; mass conservation
$\sum_{i\in E}q_i=\sum_{i\in X}p_i$ (both equal the moved mass) makes the two
coefficients equal, so $\|\dot u\|_2^2=c/(s(1-s))$ and equality holds.
\end{proof}

\begin{corollary}[Sharp, uniform in $k$]
\label{apx:cor:cat}\label{app:cor-cat}
For every $k\ge2$ and all distinct $p,q\in\Delta_{k-1}$, $R(p,q)<\pi/(2\sqrt2)$.
Hence $R^{\ast}(\Delta_{k-1})\le\pi/(2\sqrt2)$, and the supremum is not attained.
\end{corollary}

\begin{proof}
Let $\theta_0$ be the central angle between the unit vectors $\alpha,\beta$ on
the unit sphere (equivalently $\alpha=z(p)/2$, $\beta=z(q)/2$), so
$\delta_{\mathrm{geo}}(p,q)=2\theta_0$ and $\theta_0\in(0,\pi/2]$ in the orthant. The chord on
the radius-$2$ sphere subtending central angle $\theta_0$ is
$\|z(q)-z(p)\|_2=2\cdot2\sin(\theta_0/2)=4\sin(\theta_0/2)$. Since
$\sin x\le x$ (strictly for $x>0$) with $x=\theta_0/2$,
\[
\|z(q)-z(p)\|_2=4\sin(\theta_0/2)\le 4\cdot\frac{\theta_0}{2}=2\theta_0=\delta_{\mathrm{geo}}(p,q),
\]
with equality only as $\theta_0\to0$. Combining with Theorem~\ref{apx:thm:chord},
\[
R(p,q)=\frac{L(p,q)}{\delta_{\mathrm{geo}}(p,q)}
\le\frac{\pi}{2\sqrt2}\cdot\frac{\|z(q)-z(p)\|_2}{\delta_{\mathrm{geo}}(p,q)}
<\frac{\pi}{2\sqrt2},
\]
independent of $k$ (strictly, because $\sin x<x$ for $x>0$ unless $p=q$).
\end{proof}

\subsection{The matching lower bound and the exact value}

\begin{theorem}[Categorical constant (Restated from Theorem~\ref{thm:categorical})]
\label{apx:thm:categorical}
For the categorical family on $k\ge3$ outcomes,
$R^{\ast}(\Delta_{k-1})=\pi/(2\sqrt2)\approx1.1107$, independent of $k$, approached
but not attained; for $k=2$, $R^{\ast}=1$.
\end{theorem}

\begin{proof}
\emph{$k=2$.} $\Delta_1$ is one-dimensional, so $R\equiv1$ by
Lemma~\ref{apx:lem:dim-one}; thus $R^{\ast}(\Delta_1)=1$.

\emph{Upper bound ($k\ge2$).} Corollary~\ref{apx:cor:cat} gives
$R^{\ast}(\Delta_{k-1})\le\pi/(2\sqrt2)$, uniformly in $k$, not attained.

\emph{Lower bound ($k\ge3$): an explicit two-outcome family.} Restrict to moves
that change only coordinates $1$ and $2$, freezing $p_j=q_j$ for $j\ge3$. Then
$m:=p_1+p_2=q_1+q_2$ is conserved along the $m$-geodesic (the total is $1$ and
the frozen mass $1-m$ is constant). In the sphere coordinates the moving pair
lies on the circle of radius $2\sqrt m$ in the $(z_1,z_2)$-plane
($z_1^2+z_2^2=4(p_1+p_2)=4m$); parametrize
$z_1=2\sqrt m\cos\phi$, $z_2=2\sqrt m\sin\phi$, $\phi\in[0,\tfrac\pi2]$, so that
$p_1=m\cos^2\phi$, $p_2=m\sin^2\phi$. Because $p_1(s)$ is affine and monotone in
$s$, $\phi$ runs monotonically from $\phi_p$ to $\phi_q$.

\emph{Numerator.} The frozen coordinates contribute nothing to the length, and
the moving part traces a circular arc of radius $2\sqrt m$, so
$L=\int|{\,\mathrm{d} z}/{\,\mathrm{d} s}|\,\mathrm{d} s=2\sqrt m\int|\dot\phi|\,\mathrm{d} s=2\sqrt m\,\psi$, where
$\psi:=|\phi_q-\phi_p|\in(0,\tfrac\pi2]$.

\emph{Denominator.} With frozen coordinates,
$\sum_i\sqrt{p_iq_i}=\sqrt{p_1q_1}+\sqrt{p_2q_2}+\sum_{j\ge3}p_j
=m(\cos\phi_p\cos\phi_q+\sin\phi_p\sin\phi_q)+(1-m)
=m\cos\psi+(1-m)=1-m(1-\cos\psi)=1-2m\sin^2(\psi/2)$.
Hence $\delta_{\mathrm{geo}}=2\arccos\big(1-2m\sin^2(\psi/2)\big)$.

\emph{The ratio.} Therefore
\[
R_2(m,\psi)=\frac{L}{\delta_{\mathrm{geo}}}
=\frac{2\sqrt m\,\psi}{2\arccos\!\big(1-2m\sin^2(\psi/2)\big)}
=\frac{\sqrt m\,\psi}{\arccos\!\big(1-2m\sin^2(\psi/2)\big)} .
\]
Now take $m\to0$. Using $\arccos(1-x)=\sqrt{2x}\,(1+O(x))$ as $x\to0$ with
$x=2m\sin^2(\psi/2)$, so $\sqrt{2x}=2\sqrt m\sin(\psi/2)$,
\[
R_2(m,\psi)\longrightarrow\frac{\sqrt m\,\psi}{2\sqrt m\sin(\psi/2)}
=\frac{\psi/2}{\sin(\psi/2)} .
\]
The function $x\mapsto x/\sin x$ is increasing on $(0,\pi/2]$, so
$(\psi/2)/\sin(\psi/2)$ is increasing in $\psi$ on $(0,\pi/2]$ and maximal at
$\psi=\pi/2$, with value
\[
\frac{\pi/4}{\sin(\pi/4)}=\frac{\pi/4}{1/\sqrt2}=\frac{\pi}{4}\sqrt2=\frac{\pi}{2\sqrt2}.
\]
(One also checks $R_2$ is decreasing in $m$ at fixed $\psi$, so the supremum over
$m\in(0,1]$ is the $m\to0$ limit.) The two-outcome family embeds in every
$\Delta_{k-1}$ for $k\ge3$, so $R^{\ast}(\Delta_{k-1})\ge\pi/(2\sqrt2)$, approached
(as $m\to0$, $\psi\to\pi/2$) but not attained.

\emph{Conclusion.} The upper and lower bounds coincide:
$R^{\ast}(\Delta_{k-1})=\pi/(2\sqrt2)$ for $k\ge3$. The two inequality steps are
tight in complementary regimes---the chord bound is an equality for every pure
transfer (Theorem~\ref{apx:thm:chord}), while $\text{chord}/\text{arc}\to1$ exactly
as the moved mass $m\to0$ (Corollary~\ref{apx:cor:cat})---so the constant is exact.
\end{proof}


\section{Class III: the scale-product rate $\Theta(\sqrt{\log d})$}
\label{apx:sec:classIII}

This appendix supplies complete proofs of the scale-product reduction
(Lemma~\ref{lem:scale-reduction}) and the scale-product rate
(Theorem~\ref{thm:class3}), together with all supporting lemmas and the matching
upper and lower bounds. In this appendix $n$ denotes the number of factors,
which coincides with the manifold dimension $d$ for a product of
one-dimensional families.
Throughout
$\|\cdot\|=\|\cdot\|_2$.

\subsection{Reduction to a scalar functional}

\begin{lemma}[Scale-product reduction (Restated from Lemma~\ref{lem:scale-reduction})]
\label{apx:lem:scale-reduction}
Let $M_n$ be a product of $n$ one-dimensional factors with expectation
coordinate $\eta_i\in(0,\infty)$ and metric $g_{\eta_i\eta_i}=\kappa/\eta_i^2$,
$\kappa>0$ common. For endpoints $\eta^p,\eta^q$ set
\[
w_i=\log\frac{\eta_i^q}{\eta_i^p},\qquad c_i=e^{w_i}-1,\qquad
u_i(s)=\frac{c_i}{1+sc_i}\ (\equiv0\text{ if }c_i=0).
\]
Then
\[
\delta_{\mathrm{geo}}(p,q)=\sqrt\kappa\,\|w\|,\qquad
L(\gamma^{(m)})=\sqrt\kappa\int_0^1\|U(s)\|\,\mathrm{d} s,\qquad
R(w)=\frac{\int_0^1\|U(s)\|\,\mathrm{d} s}{\|w\|},
\]
with $U=(u_i)$; in particular the reduced functional $R(w)$ is independent of
$\kappa$.
\end{lemma}

\begin{proof}
\emph{Distance.} In each factor put $y_i=\sqrt\kappa\,\log\eta_i$; then
$\,\mathrm{d} y_i=\sqrt\kappa\,\,\mathrm{d}\eta_i/\eta_i$ and
$\,\mathrm{d} y_i^2=\kappa\,\,\mathrm{d}\eta_i^2/\eta_i^2=g_{\eta_i\eta_i}\,\mathrm{d}\eta_i^2$, so the
metric becomes Euclidean in $y$. Hence
$\delta_{\mathrm{geo}}(p,q)=\|y^q-y^p\|=\sqrt\kappa\,\|\log(\eta^q/\eta^p)\|=\sqrt\kappa\,\|w\|$.

\emph{Length.} Along the $m$-geodesic $\eta_i(s)=(1-s)\eta_i^p+s\eta_i^q
=\eta_i^p\big(1+s c_i\big)$ (since $\eta_i^q/\eta_i^p=e^{w_i}=1+c_i$). Then
$\dot\eta_i/\eta_i=c_i/(1+sc_i)=u_i(s)$, and the squared Fisher speed is
$\sum_i g_{\eta_i\eta_i}\dot\eta_i^2=\sum_i\kappa(\dot\eta_i/\eta_i)^2
=\kappa\|U(s)\|^2$. Hence
$L(\gamma^{(m)})=\int_0^1\sqrt\kappa\,\|U(s)\|\,\mathrm{d} s=\sqrt\kappa\int_0^1\|U\|\,\mathrm{d} s$.
The common $\sqrt\kappa$ cancels in $R=L/\delta_{\mathrm{geo}}$.

\emph{Consistency.} $\int_0^1u_i\,\mathrm{d} s=\int_0^1\frac{c_i}{1+sc_i}\,\mathrm{d} s
=\log(1+c_i)=\log(\eta_i^q/\eta_i^p)=w_i$, so $\int_0^1U\,\mathrm{d} s=w$ (used below).
\end{proof}

For the analysis we change variables. With $s=e^{-t}$ (so $t=\log(1/s)$,
$t\in[0,\infty)$) define
\[
\phi_i(t):=u_i(e^{-t})\,e^{-t}=\frac{c_i e^{-t}}{1+e^{-t}c_i}=\frac{c_i}{c_i+e^t},
\qquad \Phi=(\phi_i),
\]
so that (by $\,\mathrm{d} s=-e^{-t}\,\mathrm{d} t$ and $\|U(s)\|\,\mathrm{d} s=\|\Phi(t)\|\,\mathrm{d} t$)
\[
I(w):=\int_0^1\|U(s)\|\,\mathrm{d} s=\int_0^\infty\|\Phi(t)\|\,\mathrm{d} t .
\]
When $w_i>0$, writing $a_i=1/c_i=1/(e^{w_i}-1)$ gives
$\phi_i(t)=1/(1+a_ie^t)$.

\subsection{Per-channel and elementary bounds}

\begin{lemma}[Per-channel bounds]
\label{apx:lem:channel}
For $w_i>0$ and all $t\ge0$, $\ \phi_i(t)\le\min(1,e^{\,w_i-t})$. If moreover
$w_i\ge2$ and $t\le w_i$, then
$\phi_i(t)\ge\beta:=\dfrac{1-e^{-2}}{2-e^{-2}}\approx0.464$.
\end{lemma}

\begin{proof}
$\phi_i=c_i/(c_i+e^t)\le1$ is clear. For the exponential bound, with
$c_i=e^{w_i}-1$,
\[
\phi_i\le e^{w_i-t}=\frac{e^{w_i}}{e^t}
\iff c_ie^t\le e^{w_i}(c_i+e^t)=(1+c_i)(c_i+e^t)=c_i+c_i^2+e^t+c_ie^t,
\]
and the right minus left is $c_i+c_i^2+e^t\ge0$, true. Hence
$\phi_i\le\min(1,e^{w_i-t})$.

For the lower bound, $\phi_i(t)=c_i/(c_i+e^t)$ is decreasing in $t$, so for
$t\le w_i$, $\phi_i(t)\ge\phi_i(w_i)=\dfrac{e^{w_i}-1}{(e^{w_i}-1)+e^{w_i}}
=\dfrac{e^{w_i}-1}{2e^{w_i}-1}=\dfrac{1-e^{-w_i}}{2-e^{-w_i}}$. The map
$x\mapsto\frac{1-x}{2-x}$ has derivative $\frac{-(2-x)+(1-x)}{(2-x)^2}
=\frac{-1}{(2-x)^2}<0$, so it is decreasing in $x=e^{-w_i}$, which is itself
decreasing in $w_i$; hence $\frac{1-e^{-w_i}}{2-e^{-w_i}}$ is increasing in
$w_i$ and, for $w_i\ge2$, is at least its value at $w_i=2$, namely
$\beta=\frac{1-e^{-2}}{2-e^{-2}}$.
\end{proof}

\begin{lemma}[Elementary bracket]
\label{apx:lem:bracket}
$1\le R(w)\le\sqrt n$ for all $w\ne0$; in particular
$R^{\ast}(M_n)<\infty$ for each $n$.
\end{lemma}

\begin{proof}
\emph{Lower.} By the vector-integral inequality and
$\int_0^1U\,\mathrm{d} s=w$ (Lemma~\ref{apx:lem:scale-reduction}),
$\int_0^1\|U\|\,\mathrm{d} s\ge\big\|\int_0^1U\,\mathrm{d} s\big\|=\|w\|$, so $R\ge1$.
\emph{Upper.} Each $u_i(s)=c_i/(1+sc_i)$ is sign-definite (sign of $c_i$, i.e.\
of $w_i$) with $\int_0^1|u_i|\,\mathrm{d} s=|\log(1+c_i)|=|w_i|$. Hence
$\int_0^1\|U\|_2\,\mathrm{d} s\le\int_0^1\|U\|_1\,\mathrm{d} s=\sum_i|w_i|=\|w\|_1\le\sqrt n\|w\|_2$,
so $R\le\sqrt n$.
\end{proof}

\subsection{Upper bound $O(\sqrt{\log n})$}

\begin{lemma}[Energy identity]
\label{apx:lem:energy}
For $w\ge0$ let $g(t):=\sum_i\min\!\big(1,e^{2(w_i-t)}\big)$. Then
\[
\int_0^\infty 2t\,g(t)\,\mathrm{d} t=\sum_i\Big(w_i^2+w_i+\tfrac12\Big).
\]
\end{lemma}

\begin{proof}
By linearity it suffices to treat one index $i$. For $t\le w_i$ the minimum is
$1$ (since $e^{2(w_i-t)}\ge1$); for $t>w_i$ it is $e^{2(w_i-t)}$. Thus
\[
\int_0^\infty 2t\min(1,e^{2(w_i-t)})\,\mathrm{d} t
=\int_0^{w_i}2t\,\mathrm{d} t+\int_{w_i}^\infty 2t\,e^{2(w_i-t)}\,\mathrm{d} t .
\]
The first integral is $w_i^2$. For the second substitute $\tau=t-w_i$:
\[
\int_0^\infty 2(\tau+w_i)e^{-2\tau}\,\mathrm{d}\tau
=2\int_0^\infty\tau e^{-2\tau}\,\mathrm{d}\tau+2w_i\int_0^\infty e^{-2\tau}\,\mathrm{d}\tau
=2\cdot\tfrac14+2w_i\cdot\tfrac12=\tfrac12+w_i .
\]
So the single-index total is $w_i^2+w_i+\tfrac12$; summing over $i$ gives the
claim.
\end{proof}

\begin{lemma}[Weighted Hardy bound]
\label{apx:lem:hardy}
Let $h\ge0$ be measurable on $(0,\infty)$ and $0<a<T$. Then
\[
\int_a^T h(t)\,\mathrm{d} t\le\Big(\int_a^T 2t\,h(t)^2\,\mathrm{d} t\Big)^{1/2}
\Big(\tfrac12\log\tfrac{T}{a}\Big)^{1/2}.
\]
\end{lemma}

\begin{proof}
Cauchy--Schwarz with the split $h=(\sqrt{2t}\,h)\cdot\frac{1}{\sqrt{2t}}$:
\[
\int_a^T h\,\mathrm{d} t\le\Big(\int_a^T 2t\,h^2\,\mathrm{d} t\Big)^{1/2}
\Big(\int_a^T\frac{\,\mathrm{d} t}{2t}\Big)^{1/2},
\]
and $\int_a^T\frac{\,\mathrm{d} t}{2t}=\tfrac12\log(T/a)$.
\end{proof}

\begin{proposition}[Upper bound]
\label{apx:prop:upper}
$\displaystyle\sup_{w\in\mathbb{R}^n\setminus\{0\}}R(w)
\le\frac{\sqrt{10}}{2}\,\sqrt{\log n}\,(1+o(1))$ as $n\to\infty$, with the
$o(1)$ uniform in $w$.
\end{proposition}

\begin{proof}
\emph{Step 1: sign reduction.} Split $P=\{i:w_i>0\}$, $\mathcal M=\{i:w_i<0\}$
(zero coordinates contribute $u_i\equiv0$ and are dropped). Pointwise
$\|U(s)\|\le\|U_P(s)\|+\|U_{\mathcal M}(s)\|$, so $I(w)\le I_P+I_{\mathcal M}$
with $I_P=\int_0^1\|U_P\|\,\mathrm{d} s$, etc. For $i\in\mathcal M$ let
$\tilde w_i=|w_i|>0$ and $\tilde u_i(s)=\frac{e^{\tilde w_i}-1}{1+s(e^{\tilde w_i}-1)}$
the velocity of the positive channel of magnitude $\tilde w_i$. We claim
$|u_i(s)|=\tilde u_i(1-s)$. Indeed, with $w_i=-a$ ($a>0$),
\[
|u_i(s)|=\frac{1-e^{-a}}{1+s(e^{-a}-1)}=\frac{1-e^{-a}}{1-s(1-e^{-a})},\qquad
\tilde u_i(1-s)=\frac{e^{a}-1}{1+(1-s)(e^{a}-1)}=\frac{e^a-1}{e^a-s(e^a-1)},
\]
and multiplying the numerator and denominator of the last expression by
$e^{-a}$ gives $\frac{1-e^{-a}}{1-s(1-e^{-a})}=|u_i(s)|$. Hence, with
$\tilde U=(\tilde u_i)_{i\in\mathcal M}$ and the measure-preserving substitution
$\sigma=1-s$,
\[
I_{\mathcal M}=\int_0^1\Big(\sum_{i\in\mathcal M}u_i(s)^2\Big)^{1/2}\,\mathrm{d} s
=\int_0^1\|\tilde U(1-s)\|\,\mathrm{d} s
=\int_0^1\|\tilde U(\sigma)\|\,\mathrm{d}\sigma=I(\tilde w),
\]
the integral of the \emph{positive} configuration $\tilde w=(\tilde w_i)_{i\in\mathcal M}$.
Thus if $I(v)\le B\|v\|$ for all $v\ge0$, then
\[
I(w)\le I(w_P)+I(\tilde w)\le B(\|w_P\|+\|\tilde w\|)\le B\sqrt2\,\|w\|,
\]
using $\|w_P\|+\|\tilde w\|\le\sqrt2\,(\|w_P\|^2+\|\tilde w\|^2)^{1/2}=\sqrt2\|w\|$.
It remains to bound $I(v)$ for $v\ge0$ with $B=\tfrac{\sqrt5}{2}\sqrt{\log n}(1+o(1))$.

\emph{Step 2: split at $|w_i|=1$.} Fix $v\ge0$ and partition
$L=\{i:v_i\ge1\}$, $S=\{i:v_i<1\}$, so $I(v)\le I_L+I_S$.

\emph{Small channels.} For $v_i\in[0,1)$, $u_i(s)=\frac{c_i}{1+sc_i}\le c_i
=e^{v_i}-1\le(e-1)v_i$ by convexity of $x\mapsto e^x-1$ on $[0,1]$ (a convex
function lies below its chord: $e^x-1\le(e^1-1)x$ for $x\in[0,1]$). Hence
$\|U_S(s)\|\le(e-1)\|v_S\|$ for every $s$, so $I_S\le(e-1)\|v_S\|\le(e-1)\|v\|$.

\emph{Large channels.} By Lemma~\ref{apx:lem:channel},
$\phi_i(t)^2\le\min(1,e^{2(v_i-t)})$ for $i\in L$, so with $g$ as in
Lemma~\ref{apx:lem:energy} (summed over $L$), $\|\Phi_L(t)\|\le\sqrt{g(t)}$ and
$I_L\le\int_0^\infty\sqrt g\,\,\mathrm{d} t$. By Lemma~\ref{apx:lem:energy} and $v_i\ge1$
(so $v_i\le v_i^2$ and $\tfrac12\le\tfrac12v_i^2$),
\[
\int_0^\infty 2t\,g\,\mathrm{d} t=\sum_{i\in L}\big(v_i^2+v_i+\tfrac12\big)
\le\sum_{i\in L}\tfrac52 v_i^2=\tfrac52\|v_L\|^2 .
\]
Let $v_{\max}=\max_{i\in L}v_i\le\|v_L\|$ and split at $a=v_{\max}/n$,
$T=v_{\max}+\tfrac12\log n$. Three pieces:
\begin{itemize}
\item \emph{Head.} $g\le|L|\le n$, so $\int_0^a\sqrt g\,\mathrm{d} t\le\sqrt n\,a
=v_{\max}/\sqrt n\le\|v_L\|/\sqrt n$.
\item \emph{Middle (Hardy).} By Lemma~\ref{apx:lem:hardy} with $h=\sqrt g$,
\[
\int_a^T\sqrt g\,\mathrm{d} t\le\Big(\int_a^T 2t\,g\,\mathrm{d} t\Big)^{1/2}
\Big(\tfrac12\log\tfrac Ta\Big)^{1/2}
\le\sqrt{\tfrac52}\,\|v_L\|\Big(\tfrac12\log\tfrac Ta\Big)^{1/2}.
\]
Here $\tfrac Ta=\dfrac{v_{\max}+\tfrac12\log n}{v_{\max}/n}
=n\Big(1+\dfrac{\log n}{2v_{\max}}\Big)$, so
$\log\tfrac Ta=\log n+\log\!\big(1+\tfrac{\log n}{2v_{\max}}\big)$. The
correction is decreasing in $v_{\max}$, hence for $v_{\max}\ge1$ is largest at
$v_{\max}=1$: $\log(1+\tfrac{\log n}{2})=O(\log\log n)$, and
$\log\log n/\log n\to0$. Thus $\log\tfrac Ta=\log n+O(\log\log n)$ uniformly, so
$(\tfrac12\log\tfrac Ta)^{1/2}=(\tfrac12\log n)^{1/2}(1+o(1))$, and since
$\sqrt{\tfrac52}\cdot(\tfrac12)^{1/2}=\tfrac{\sqrt5}{2}$, this term is
$\tfrac{\sqrt5}{2}\|v_L\|\sqrt{\log n}(1+o(1))$, uniformly in $v$.
\item \emph{Tail.} For $t\ge T$, $g(t)\le|L|e^{2(v_{\max}-t)}$, so
$\int_T^\infty\sqrt g\,\mathrm{d} t\le\sqrt{|L|}\,e^{v_{\max}-T}=\sqrt{|L|}\,e^{-\tfrac12\log n}
=\sqrt{|L|}/\sqrt n\le1$.
\end{itemize}
Collecting,
\[
I_L\le\frac{\sqrt5}{2}\|v_L\|\sqrt{\log n}(1+o(1))+\frac{\|v_L\|}{\sqrt n}+1
=\frac{\sqrt5}{2}\|v_L\|\sqrt{\log n}(1+o(1)),
\]
uniformly, since the last two terms are lower order relative to
$\sqrt{\log n}\,\|v_L\|$ (indeed $\|v_L\|\ge1$ when $L\ne\emptyset$).

\emph{Step 3: assemble.} With $\|v_L\|\le\|v\|$ and writing the small term as a
vanishing fraction of the main one,
\[
(e-1)\|v_S\|\le(e-1)\|v\|
=\underbrace{\frac{2(e-1)}{\sqrt5\sqrt{\log n}}}_{\to0}\cdot
\frac{\sqrt5}{2}\sqrt{\log n}\,\|v\|,
\]
so $I(v)\le\frac{\sqrt5}{2}\sqrt{\log n}\,\|v\|(1+o(1))$ for all $v\ge0$,
uniformly. The sign reduction of Step~1 multiplies by $\sqrt2$, giving
$R(w)=I(w)/\|w\|\le\frac{\sqrt{10}}{2}\sqrt{\log n}(1+o(1))$, uniformly in $w$.
\end{proof}

\begin{remark}[the cutoffs are not optimized]
The values $a=v_{\max}/n$ and $T=v_{\max}+\tfrac12\log n$ are chosen only to push
head and tail below the middle term. Any $a=v_{\max}/n^\gamma$ with $\gamma\ge1$
keeps $\int_0^a\sqrt g\le v_{\max}n^{1/2-\gamma}$ small and rescales the constant
by $\sqrt\gamma$ through $\log(T/a)=\gamma\log n+O(\log\log n)$; any
$T=v_{\max}+c\log n$ with $c\ge\tfrac12$ keeps
$\int_T^\infty\sqrt g\le n^{1/2-c}\le1$. Only the $n$-scaling is pinned; the
leading constant is not sharp, and the matching lower bound below supplies the
$\Omega(\sqrt{\log n})$ that certifies $\Theta$.
\end{remark}

\subsection{Lower bound $\Omega(\sqrt{\log n})$}

\begin{proposition}[Lower bound with padding]
\label{apx:prop:lower}
For every $n\ge1$, let $m=\lfloor\log_4(3n+1)\rfloor$. There is $w\in\mathbb{R}^n$ with
\[
R(w)\ge\frac{\beta}{2}\sqrt m=\Omega(\sqrt{\log n}),\qquad
\beta=\frac{1-e^{-2}}{2-e^{-2}} .
\]
Consequently $\sup_wR(w)=\Omega(\sqrt{\log n})$ for all $n$.
\end{proposition}

\begin{proof}
Let $N_m=(4^m-1)/3\le n$ (this is $\sum_{k=1}^{m}4^{m-k}$). On the first $N_m$
coordinates place $w_k=2^k$ with multiplicity $n_k=4^{m-k}$ ($k=1,\dots,m$),
and set the remaining $n-N_m$ coordinates to zero. Inactive coordinates have
$u_i\equiv0$ and change neither numerator nor denominator.

\emph{Denominator.} $\|w\|^2=\sum_{k=1}^m n_k\,w_k^2=\sum_{k=1}^m 4^{m-k}4^k
=\sum_{k=1}^m 4^m=m4^m$, so $\|w\|=2^m\sqrt m$.

\emph{Numerator.} Bound $I=\int_0^\infty\|\Phi\|\,\mathrm{d} t$ below by integrating only
over the disjoint windows $W_j=(2^{j-1},2^j]$, $|W_j|=2^{j-1}$, $j=1,\dots,m$.
For $t\in W_j$ and the block $k=j$: $w_j=2^j\ge t$ (since $t\le2^j$) and
$w_j=2^j\ge2$, so by Lemma~\ref{apx:lem:channel} every one of the $n_j=4^{m-j}$
channels in block $j$ has $\phi_j(t)\ge\beta$. Keeping only that block,
\[
\|\Phi(t)\|\ge\beta\sqrt{n_j}=\beta\,2^{m-j},\qquad
\int_{W_j}\|\Phi\|\,\mathrm{d} t\ge\beta\,2^{m-j}\cdot2^{j-1}=\beta\,2^{m-1}.
\]
Summing over $j=1,\dots,m$, $I\ge m\beta2^{m-1}$, and therefore
\[
R(w)=\frac{I}{\|w\|}\ge\frac{m\beta2^{m-1}}{2^m\sqrt m}=\frac{\beta}{2}\sqrt m .
\]
Since $m=\lfloor\log_4(3n+1)\rfloor=\Theta(\log n)$, this is
$\Omega(\sqrt{\log n})$. A zero coordinate is realized by equal source and
target scales (an added unit relative eigenvalue in the $\mathrm{SPD}$ picture), proving
that padding with inactive (zero-$w_i$) coordinates does not affect the lower bound.
\end{proof}

\subsection{The rate theorem}

\begin{theorem}[Scale-product rate (Restated from Theorem~\ref{thm:class3})]
\label{apx:thm:class3}
Let $M_n$ be a product of $n$ identical one-dimensional scale families. Then
$R^{\ast}(M_n)=\Theta(\sqrt{\log n})$; in particular $\sup_nR^\ast(M_n)=\infty$.
The same rate holds for the Gaussian covariance manifold $\mathrm{SPD}(n)$ under the
affine-invariant metric.
\end{theorem}

\begin{proof}
By Lemma~\ref{apx:lem:scale-reduction}, $R^{\ast}(M_n)=\sup_{w\ne0}R(w)$ for the
scalar functional $R(w)=\int_0^1\|U\|\,\mathrm{d} s/\|w\|$. Lemma~\ref{apx:lem:bracket} gives
finiteness at each $n$; Proposition~\ref{apx:prop:upper} gives $O(\sqrt{\log n})$ and
Proposition~\ref{apx:prop:lower} gives $\Omega(\sqrt{\log n})$, so
$\sup_wR(w)=\Theta(\sqrt{\log n})$, and the diverging lower bound gives
$\sup_nR^\ast(M_n)=\infty$.

For $\mathrm{SPD}(n)$: the affine-invariant metric
$g_\Sigma(A,B)=\tfrac12\operatorname{tr}(\Sigma^{-1}A\Sigma^{-1}B)$ is invariant under
congruence $\Sigma\mapsto G\Sigma G^\top$ \cite{wells2020covariance,skovgaard1984riemannian},
so we may take $\Sigma_p=I$, $\Sigma_q=\mathrm{diag}(r_1,\dots,r_n)$ with $r_i$
the relative eigenvalues. The $m$-geodesic
$\Sigma(s)=(1-s)I+s\,\mathrm{diag}(r_i)$ is diagonal and decouples; with
$w_i=\log r_i$ the closed forms
\[
\delta_{\mathrm{geo}}=\frac1{\sqrt2}\|w\|,\qquad L(\gamma^{(m)})=\frac1{\sqrt2}\int_0^1\|U\|\,\mathrm{d} s,
\qquad R=\frac{\int_0^1\|U\|\,\mathrm{d} s}{\|w\|}
\]
hold (each factor is a scale family with $g_{\eta\eta}=\tfrac12/\eta^2$, i.e.\
$\kappa=\tfrac12$; the $\sqrt2$'s are this $\kappa$ and cancel in $R$ by
Lemma~\ref{apx:lem:scale-reduction}). Every $w\in\mathbb{R}^n$ arises from some pair, so
$R^{\ast}(\mathrm{SPD}(n))=\sup_wR(w)=\Theta(\sqrt{\log n})$.
\end{proof}

\begin{remark}[diameter, not curvature]
$\mathrm{SPD}(n)$ is Hadamard (negatively curved) and infinite-diameter; a Gamma product
is Riemann-flat (zero curvature) and infinite-diameter. Both are
$\Theta(\sqrt{\log n})$, while the positively curved but finite-diameter simplex
is bounded (Section~\ref{apx:sec:categorical}). Curvature takes every sign within the
bounded and unbounded regimes; finite versus infinite Fisher--Rao diameter is
what the sufficient criterion of Theorem~\ref{apx:thm:bounded} keys on.
\end{remark}


\section{Operational cost: a one-sided step-count bound}
\label{app:opcost}

This appendix proves the one-sided budgeted step-count result
(Proposition~\ref{prop:steps}) via the exactness and shadowing results developed
below (Corollary~\ref{apx:cor:exact} and Lemma~\ref{apx:lem:shadow}). The
matching \emph{upper} bound on traversed length for inexact schemes remains open
(Remark~\ref{apx:rem:opengap}); we prove only the one-sided direction used by
the proposition.

\paragraph{Local setting.}
Fix endpoints $p,q$; write $\eta^\star=\eta(q)$, $h=\eta^\star-\eta(p)$, and let
$\gamma$ be the $m$-geodesic $\eta(\gamma(s))=\eta(p)+sh$, with speed
$\varphi(s)=\|h\|_{\gamma(s)}$ and length $L=\int_0^1\varphi$. Fix a compact
convex $\eta$-slab $U\supset\gamma([0,1])$ in the domain and set
\[
\kappa_U:=\sup_{x\in U}\sup_{v\ne0}\frac{|T_x(v,v,v)|}{\|v\|_x^3}<\infty
\]
(finite: $U$ compact, $T$ continuous; $\kappa_U\le\kappa_T(M)$ when the latter is
finite). Let $\Delta_U$ be the Fisher diameter of $U$ and $B:=e^{\kappa_U\Delta_U/2}$.
Restricting Lemma~\ref{apx:lem:ruler} to $U$ (its proof is pointwise, so $\kappa_T$ may be
replaced by $\kappa_U$) gives, for every fixed $\eta$-coordinate vector $a$ and
$x,y\in U$, $\|a\|_y\le e^{\kappa_U\delta_{\mathrm{geo}}(x,y)/2}\|a\|_x$. Three consequences,
along straight $\eta$-segments in the convex $U$, are used repeatedly:
\begin{itemize}
\item[(a)] \emph{Speed comparison:} $\varphi(s)/\varphi(s')\le e^{\kappa_U L/2}$;
since $\int_0^1\varphi=L$ forces $\min\varphi\le L$, also
$\bar\varphi:=\max\varphi\le Le^{\kappa_U L/2}$.
\item[(b)] \emph{Arc versus chord:} a straight $\eta$-segment of chord
$\ell=\|\eta(y)-\eta(x)\|_x$ has Fisher arc length $\mathfrak a$ with
$|\mathfrak a-\ell|\le\kappa_U\ell^2$ whenever $\kappa_U\ell\le\tfrac14$.
(The arc function $\xi(\tau)$ obeys $\xi'\le\ell\,e^{\kappa_U\xi/2}$, whence
$\mathfrak a\le\ell+\kappa_U\ell^2/2$ by ODE comparison; and
$\mathfrak a\ge\ell e^{-\kappa_U\mathfrak a/2}\ge\ell-\kappa_U\ell^2$.)
\item[(c)] \emph{Norm transport:} $\|a\|_y\le B\|a\|_x$ for $x,y\in U$.
\end{itemize}

\subsection{The exact trust-region step stays on the $m$-geodesic}

\begin{lemma}[Exactness]
\label{apx:lem:exact}\label{app:lem-exact}
Fix the current iterate $x_k$ and let $x_{k+1}$ minimize the forward-KL objective
$J(x)=\mathrm{KL}(p_q\Vert p_x)$ over the trust region
$\{x:\mathrm{KL}(p_{x_k}\Vert p_x)\le\delta\}$. Then $\eta(x_{k+1})$ lies on the
$\eta$-segment $[\eta(x_k),\eta^\star]$. The same holds for the proximal step
$\arg\min_x J(x)+\tfrac1{2\mu}\mathrm{KL}(p_{x_k}\Vert p_x)$, $\mu>0$.
\end{lemma}

\begin{proof}
In natural coordinates, using the Bregman/KL form (F1),
$J(\theta)=\mathrm{KL}(p_q\Vert p_\theta)=\psi(\theta)-\langle \theta,\eta^\star \rangle+\text{const}$
and $C(\theta):=\mathrm{KL}(p_{\theta_k}\Vert p_\theta)
=\psi(\theta)-\psi(\theta_k)-\langle \theta-\theta_k,\eta_k \rangle$, both convex
($\psi$ convex). Crucially, their gradients are \emph{exact} $\eta$-differences
with no remainder (the gift of dual flatness, F1):
\[
\nabla J(\theta)=\nabla\psi(\theta)-\eta^\star=\eta(\theta)-\eta^\star,\qquad
\nabla C(\theta)=\nabla\psi(\theta)-\eta_k=\eta(\theta)-\eta_k .
\]
Slater's condition holds ($C(\theta_k)=0<\delta$), so a minimizer of the convex
program satisfies the KKT stationarity condition: there is a multiplier
$\lambda\ge0$ with $\nabla J(\eta_{k+1})+\lambda\nabla C(\eta_{k+1})=0$, i.e.
\[
(\eta_{k+1}-\eta^\star)+\lambda(\eta_{k+1}-\eta_k)=0
\ \Longrightarrow\ \eta_{k+1}=\frac{\eta^\star+\lambda\,\eta_k}{1+\lambda},
\]
a convex combination of $\eta_k$ and $\eta^\star$ (if the constraint is inactive
then $\lambda=0$ and $\eta_{k+1}=\eta^\star$). Thus $\eta_{k+1}\in[\eta_k,\eta^\star]$.
The proximal step gives the same stationarity equation with $\lambda=1/(2\mu)$.
\end{proof}

Started at $p$, every exact iterate therefore lies on $[\eta(p),\eta^\star]$,
each edge is a sub-segment of $\gamma$, and the subdivision advances
monotonically toward $\eta^\star$.

\begin{corollary}[Exact step count]
\label{apx:cor:exact}\label{app:cor-exact}
Run the exact scheme with chord cap $\ell_k=\|\eta_{k+1}-\eta_k\|_{x_k}\le\varepsilon$,
terminating at the first $N$ with $\|\eta_N-\eta^\star\|_{x_N}\le\varepsilon$;
suppose $\kappa_U\varepsilon\le\tfrac14$. Then
$\Lambda_{\mathrm{NGD}}=L-\mathrm{tail}$ with $0\le\mathrm{tail}\le2\varepsilon$,
and
\[
N\ge\frac{\Lambda_{\mathrm{NGD}}}{\varepsilon(1+\kappa_U\varepsilon)}
\ge\frac{L}{\varepsilon}\Big(1-\big(\kappa_U+\tfrac2L\big)\varepsilon\Big).
\]
\end{corollary}

\begin{proof}
The edges are sub-arcs of $\gamma$, so
$\Lambda_{\mathrm{NGD}}=\int_0^{\sigma_N}\varphi$ where $\gamma(\sigma_N)=x_N$.
For the tail, set $\rho:=(1-\sigma_N)\varphi(\sigma_N)=\|\eta_N-\eta^\star\|_{x_N}
\le\varepsilon$. By the speed comparison (a),
$\varphi(s)\le\varphi(\sigma_N)e^{\kappa_U\,\mathrm{arc}(\sigma_N,s)/2}$, so
$\mathrm{tail}=\int_{\sigma_N}^1\varphi\le\rho\,e^{\kappa_U\,\mathrm{tail}/2}$;
ODE comparison gives $\mathrm{tail}\le2\rho\le2\varepsilon$ for
$\kappa_U\rho\le\tfrac12$. Each edge's arc length is at most
$\ell_k(1+\kappa_U\varepsilon)\le\varepsilon(1+\kappa_U\varepsilon)$ by
consequence (b), so summing the $N$ edges,
$N\varepsilon(1+\kappa_U\varepsilon)\ge\Lambda_{\mathrm{NGD}}$, which is the first
inequality. The second uses $\Lambda_{\mathrm{NGD}}=L-\mathrm{tail}\ge L-2\varepsilon$
and $1/(1+\kappa_U\varepsilon)\ge1-\kappa_U\varepsilon$.
\end{proof}

\begin{remark}[the KL orientation matters]
Exactness is special to the orientation with the current iterate on the left,
$\mathrm{KL}(p_{x_k}\Vert p_{x_{k+1}})\le\delta$ (as TRPO/NPG use). With the reversed
constraint $\mathrm{KL}(p_x\Vert p_{x_k})\le\delta$ one has
$\nabla\widetilde C(\theta)=g(\theta)(\theta-\theta_k)\ne\eta(\theta)-\eta_k$,
and the minimizer leaves the segment by $O(\kappa_U\ell^2)$---precisely the
regime of Definition~\ref{apx:def:consistent}.
\end{remark}

\subsection{Shadowing for consistent approximate schemes}

\begin{definition}[$\beta$-consistent budgeted scheme]
\label{apx:def:consistent}\label{app:def-consistent}
Iterates $x_0=p,x_1,\dots$ with $\eta_k:=\eta(x_k)$ of the form
$\eta_{k+1}=\eta_k-t_k(\eta_k-\eta^\star)+e_k$, $t_k\in(0,1]$, subject to
(i) \emph{budget} $\ell_k:=\|\eta_{k+1}-\eta_k\|_{x_k}\le\varepsilon$;
(ii) \emph{consistency} $\|e_k\|_{x_k}\le\beta\ell_k^2$;
(iii) \emph{termination} the run stops at the first $N$ with
$\|\eta_N-\eta^\star\|_{x_N}\le\varepsilon$.
\end{definition}

Here $e_k$ is the full deviation of step $k$ from the exact natural-gradient ray;
(ii) says the scheme is a first-order-consistent integrator of
$\dot\eta=-(\eta-\eta^\star)$. The standard discretization qualifies: the
natural-parameter update
$\theta_{k+1}=\theta_k-t_k g(\theta_k)^{-1}(\eta_k-\eta^\star)$ has, by Taylor's
theorem with integral remainder applied to $\eta(\theta)$ (whose second
derivative is $\partial^2\eta/\partial\theta^2=\partial^3\psi=T$, F5),
\[
e_k=\int_0^1(1-u)\,T_{\theta_k+u\Delta\theta_k}(\Delta\theta_k,\Delta\theta_k,\cdot)^{\sharp}\,\mathrm{d} u,
\]
so by Lemma~\ref{apx:lem:polar} and Lemma~\ref{apx:lem:ruler},
$\|e_k\|_{x_k}\le\kappa_U\ell_k^2$ for $\kappa_U\varepsilon\le\tfrac14$; the
scheme is $\beta$-consistent with $\beta\le\kappa_U$ (the reversed-orientation
trust-region step is covered by the same computation).

\begin{lemma}[Shadowing]
\label{apx:lem:shadow}\label{app:lem-shadow}
Run a $\beta$-consistent budgeted scheme in $U$. There are explicit
$\varepsilon_0>0$ and $C_1,C_2,C_3<\infty$ depending only on
$(\kappa_U,\beta,L,B)$ (and, through $\varepsilon_0$, on the slab width and
$\varphi(0)=\|h\|_p$) such that for all $\varepsilon\le\varepsilon_0$:
(i) every iterate has Fisher distance to $\gamma$ at most $C_1\varepsilon$;
(ii) if the run terminates, $\Lambda_{\mathrm{NGD}}\ge L-C_2\varepsilon$;
(iii) hence $N\ge(L/\varepsilon)(1-C_3\varepsilon)$. Admissible constants:
$C_1=6B^3\beta Le^{\kappa_U L/2}$, $C_3=C_2/L+2\kappa_U$, and
\[
C_2=4+16B^3\beta Le^{\kappa_U L/2}
+L\big[\kappa_U C_1+3\kappa_U+(1+B^2)\beta+3Be^{\kappa_U L/2}\beta\big].
\]
\end{lemma}

\begin{proof}
Write $\|\cdot\|_0:=\|\cdot\|_p$ and split the $\eta$-space as
$\mathrm{span}(h)\oplus W$, with $W$ the $g_p$-orthogonal complement; every
vector $v=\lambda_v h+Pv$ has $|\lambda_v|\|h\|_0\le\|v\|_0$ and
$\|Pv\|_0\le\|v\|_0$. Set $\eta_k=\eta(p)+\sigma_k h+\vec d_k$, $\vec d_k\in W$.
The recursion $\eta_{k+1}=\eta_k-t_k(\eta_k-\eta^\star)+e_k$ splits along
$\mathrm{span}(h)$ and $W$ (using $\eta^\star-\eta(p)=h$, so
$\eta_k-\eta^\star=(\sigma_k-1)h+\vec d_k$):
\[
\sigma_{k+1}=\sigma_k+t_k(1-\sigma_k)+\lambda_{e_k},\qquad
\vec d_{k+1}=(1-t_k)\vec d_k+Pe_k .
\]
Norm transport (c) converts base points at cost $B$; in particular
$\|e_k\|_0\le B\beta\ell_k^2$. We argue under the bootstrap that all iterates and
edges so far lie in $U$; for $\varepsilon\le\varepsilon_0$ the bounds below are
strictly inside the slab, so a finite induction closes the bootstrap
(suppressed).

\emph{Step A (deviation from the ray).} Let $D_k:=\|\vec d_k\|_0$. From the
$W$-recursion and $t_k\in(0,1]$, $D_{k+1}\le(1-t_k)D_k+B\beta\ell_k^2$. From the
step definition $\ell_k\le t_k\rho_k+\beta\ell_k^2$ (with
$\rho_k:=\|\eta_k-\eta^\star\|_{x_k}$) and $\beta\varepsilon\le\tfrac14$, we get
$\ell_k\le\tfrac43 t_k\rho_k$, so $B\beta\ell_k^2\le t_k\cdot\tfrac43B\beta\varepsilon\rho_k$.
Each update is then a convex combination
$D_{k+1}\le\max(D_k,\tfrac43B\beta\varepsilon\rho_k)$; with
$\rho_k\le\tfrac54\|h\|_{x_k}+\|\vec d_k\|_{x_k}\le\tfrac54BLe^{\kappa_U L/2}+BD_k$
(speed comparison (a) and transport (c)) and the bootstrap $D_k\le L$,
\[
D_{\max}\le\tfrac43B\beta\varepsilon\cdot3BLe^{\kappa_U L/2}
=4B^2\beta Le^{\kappa_U L/2}\varepsilon .
\]
The connecting segment from $\gamma(\sigma_k)$ to $\eta_k$ has chord $\le BD_k$
and arc $\le\tfrac43BD_{\max}$, giving (i) with $C_1=6B^3\beta Le^{\kappa_U L/2}$.
Telescoping the $W$-recursion, $\sum_k t_kD_k\le D_0+B\beta\sum_k\ell_k^2
\le B\beta\varepsilon\Lambda_c$, where $\Lambda_c:=\sum_k\ell_k$: the contraction
pays for the whole sum with no logarithmic factor.

\emph{Step B (arc versus chord).} By consequence (b), the traversed Fisher
length $\Lambda=\sum_k\mathfrak a_k$ satisfies
$|\Lambda-\Lambda_c|\le\kappa_U\varepsilon\Lambda_c$.

\emph{Step C (no corner-cutting).} Let $N':=\min\{N,\ \text{first }k\text{ with }
\sigma_k>1\}$. With $u_k:=h/\|h\|_{x_k}$, Cauchy--Schwarz gives
$\ell_k\ge\langle \Delta\eta_k,u_k \rangle_{x_k}
\ge t_k(1-\sigma_k)\|h\|_{x_k}-t_kBD_k-\beta\ell_k^2$. Summing over $k<N'$ and
using Step~A's telescoped sum,
\[
\Lambda_c\ge\sum_{k<N'}t_k(1-\sigma_k)\|h\|_{x_k}-(1+B^2)\beta\varepsilon\Lambda_c .
\]
Writing $t_k(1-\sigma_k)=\Delta\sigma_k-\lambda_{e_k}$ with
$\sum_k|\lambda_{e_k}|\le B\beta\varepsilon\Lambda_c/\varphi(0)$, and
$\|h\|_{x_k}\ge(1-\kappa_U C_1\varepsilon)\varphi(\sigma_k)$ by ruler distortion
along the connecting segment, the Riemann sum
$\sum\Delta\sigma_k\varphi(\sigma_k)$ compares to $\int_0^{\sigma_{N'}}\varphi$ at
relative cost $2\kappa_U\varepsilon$, with backtracking bounded by
$\bar\varphi\sum|\lambda_{e_k}|\le Be^{\kappa_U L/2}\beta\varepsilon\Lambda_c$.
The uncovered tail at $N'$ is either empty ($\sigma_{N'}>1$) or, if
$\rho_{N'}\le\varepsilon$, obeys
$(1-\sigma_{N'})\varphi(\sigma_{N'})\le2(\varepsilon+BD_{\max})$, whence by the
ODE comparison of Corollary~\ref{apx:cor:exact},
$\mathrm{tail}\le4(\varepsilon+BD_{\max})\le(4+16B^3\beta Le^{\kappa_U L/2})\varepsilon$.
Collecting the factors (each of the form $1-c\varepsilon$, with one division by
$1+c'\varepsilon$) yields $\Lambda\ge L-C_2\varepsilon$ with the stated $C_2$.
Part (iii) follows from
$N\ge N'\ge\Lambda^{<N'}/(\varepsilon(1+\kappa_U\varepsilon))$. The threshold
$\varepsilon_0$ collects the finitely many smallness constraints used
($\kappa_U\varepsilon\le\tfrac14$, $\beta\varepsilon\le\tfrac14$,
$C_1\varepsilon\le\tfrac12\times$ slab width, the two bounds keeping
$\sigma_k\le\tfrac54$ and $D_k\le L$, and
$\kappa_U\varepsilon(1+4B^3\beta Le^{\kappa_U L/2})\le\tfrac14$).
\end{proof}

\begin{remark}[the remaining gap, stated not papered over]
\label{apx:rem:opengap}
Lemma~\ref{apx:lem:shadow} is deliberately one-sided: it proves
$\Lambda\ge L(1-O(\varepsilon))$, the direction Proposition~\ref{apx:prop:steps}
consumes (NGD cannot cut the corner). The matching \emph{upper} bound
$\Lambda\le L(1+O(\varepsilon))$ for inexact schemes---ruling out a consistent
scheme that wanders and traverses \emph{more} than $L$---is trivially true for
the exact scheme ($\Lambda\le L$), numerically clean for $\theta$-Euler, and
expected from the same bookkeeping plus a mild step-size regularity hypothesis
(e.g.\ $t_k$ bounded away from $1$, or monotone $\sigma_k$). \textbf{It is not
proved here, and nothing above relies on it.} This is the sole reason
Proposition~\ref{apx:prop:steps} is a lower, not two-sided, bound.
\end{remark}

\subsection{The step-count proposition}

\begin{proposition}[One-sided budgeted step count (Restated from Proposition~\ref{prop:steps})]
\label{apx:prop:steps}
Fix $p,q\in M$ and run NGD on the forward-KL objective toward $q$ under the
per-step budget $\ell_k\le\varepsilon$, using the exact trust-region (or
proximal) step, or any $\beta$-consistent budgeted scheme
(Definition~\ref{apx:def:consistent}) with $\varepsilon\le\varepsilon_0$, until
$\|\eta_N-\eta^\star\|_{x_N}\le\varepsilon$. Then
\[
N_{\mathrm{NGD}}\ge\frac{L(\gamma^{(m)})}{\varepsilon}\big(1-O(\varepsilon)\big)
=R(p,q)\,\frac{\delta_{\mathrm{geo}}(p,q)}{\varepsilon}\big(1-O(\varepsilon)\big).
\]
For fixed distinct $p,q$, letting $N_{\mathrm{geo}}(\varepsilon)$ be the minimum number
of Fisher-chord-capped steps inscribing a minimizing geodesic (so
$N_{\mathrm{geo}}(\varepsilon)=(\delta_{\mathrm{geo}}(p,q)/\varepsilon)(1+O(\varepsilon))$), this gives
$N_{\mathrm{NGD}}/N_{\mathrm{geo}}(\varepsilon)\ge R(p,q)(1-O(\varepsilon))$.
\end{proposition}

\begin{proof}
Two facts combine. First, under the chord cap the Fisher arc length of an edge
satisfies $\mathfrak a_k=\ell_k+O(\ell_k^2)$ uniformly on $U$ (consequence (b)),
so an inscribed geodesic comparator has
$N_{\mathrm{geo}}(\varepsilon)=(\delta_{\mathrm{geo}}(p,q)/\varepsilon)(1+O(\varepsilon))$ (the integer
and arc--chord errors absorbed into the relative term). Second, NGD's iterates
track the $m$-geodesic and hence traverse Fisher length
$\Lambda_{\mathrm{NGD}}\ge L(\gamma^{(m)})(1-O(\varepsilon))$: for the exact
scheme this is Corollary~\ref{apx:cor:exact} ($\Lambda=L-\mathrm{tail}$,
$\mathrm{tail}\le2\varepsilon$), and for a $\beta$-consistent scheme it is
Lemma~\ref{apx:lem:shadow}(ii) ($\Lambda\ge L-C_2\varepsilon$). Since each edge has
arc $\le\varepsilon(1+\kappa_U\varepsilon)$,
\[
N_{\mathrm{NGD}}\ge\frac{\Lambda_{\mathrm{NGD}}}{\varepsilon(1+\kappa_U\varepsilon)}
\ge\frac{L(\gamma^{(m)})}{\varepsilon}\big(1-O(\varepsilon)\big).
\]
Substituting $L(\gamma^{(m)})=R(p,q)\delta_{\mathrm{geo}}(p,q)$ gives the displayed equality, and
dividing by $N_{\mathrm{geo}}(\varepsilon)$ gives the ratio statement.
\end{proof}

\begin{remark}[Two boundary remarks]
\emph{Fixed-rate NGD escapes $R$.} With a constant learning rate,
$\eta_{k+1}=\eta_k-h_0(\eta_k-\eta^\star)$ gives
$\eta_k-\eta^\star=(1-h_0)^k(\eta_0-\eta^\star)$, so the step count to a fixed
tolerance depends on the endpoints and $h_0$ but not on arc length. There is no
contradiction: fixed-rate steps violate the Fisher budget---their chords
$\ell_k=h_0\|\eta_k-\eta^\star\|_{x_k}$ decay geometrically (large early, tiny
late). $R$ governs step count only under the formal Fisher-chord budget.
\emph{The comparator is an oracle.} ``Geodesic descent'' inscribes
Fisher-length-$\le\varepsilon$ steps along the true minimizing geodesic; its
count lower-bounds that of every budgeted method
($\Lambda\ge\delta_{\mathrm{geo}}$ for any path). Because Fisher--Rao geodesics are in general
intractable to compute, Proposition~\ref{apx:prop:steps} is a one-sided statement
about path cost against an oracle, not a claim that a practical competitor is
$R$ times faster.
\end{remark}

\bigskip
\hrule
\bigskip

\noindent\emph{These appendices collect the complete theorem, lemma, proposition, and
corollary proofs used in the main text. External facts
used: Legendre/Bregman duality of exponential families (F1--F5); Banach's
theorem on symmetric multilinear forms \cite{banach1938homogene}; the
self-concordance inequality \cite{nesterov1994interior}; congruence invariance
and closed forms of the affine-invariant $\mathrm{SPD}$ metric
\cite{wells2020covariance,skovgaard1984riemannian}; the categorical
square-root/sphere isometry \cite{miyamoto2024closedform,nielsen2023simple};
Cauchy--Schwarz, the Beta integral $\int_0^1(s(1-s))^{-1/2}\,\mathrm{d} s=\pi$,
$\int_0^\infty(ax^2+b)^{-1}\,\mathrm{d} x=\pi/(2\sqrt{ab})$, $\sin x\le x$, Taylor's
theorem with integral remainder, KKT with Slater's condition, and elementary ODE
(Gr\"onwall-type) comparison. The only result left unproved is the matching
upper bound on traversed length for inexact schemes
(Remark~\ref{apx:rem:opengap}), which the paper also flags as open.}

\section{Related Work}
\label{apx:sec:related-work}

Our object is the Fisher dilation of a particular distinguished curve: the
mixture ($m$-) geodesic. Neither ingredient is new in isolation. Nesterov and
Todd formalized subarc-uniform length-to-distance dilation for prescribed
curves in Hessian Riemannian geometry \cite{nesterov2002riemannian}, while
Fujiwara and Amari proved that canonical-divergence gradient flows follow dual
affine geodesics \cite{fujiwara1995gradient}. Our contribution is instead to
determine exact constants and dimension-dependent worst-case rates for mixture
geodesics in concrete statistical families, and to give this dilation a
restricted natural-gradient interpretation.

\paragraph{Dual geometry, canonical-divergence flows, and natural gradient.}
Fujiwara and Amari's Theorem~1 identifies canonical-divergence gradient flows
with the appropriate dual affine geodesics, and their Corollary~1 specializes
the result to maximum likelihood in exponential families
\cite{fujiwara1995gradient}. Thus, the forward-KL natural-gradient trajectory's
mixture-affine form is prior geometry, not a discovery of this paper. Felice
and Ay later formulated canonical-divergence dynamics through Lagrangian and
Hamiltonian systems \cite{felice2018dynamical}. From an optimization
perspective, Raskutti and Mukherjee's Theorem~1 identifies mirror descent with
natural gradient on the dual Hessian manifold and interprets the usual
natural-gradient step as a first-order retraction \cite{raskutti2015information}.
Omiya and Komaki make dual-connection geodesic descent explicit: their
Proposition~3.2 relates mirror descent to a dual-geodesic update, and their
Theorem~4.1 and Corollary~4.2 give exact affine-coordinate displacements for
canonical/KL divergences \cite{omiya2025minimization}. They propose geodesic
updates; we instead measure how much Fisher length one such choice---the
mixture geodesic---can spend relative to a shortest Fisher path. Amari's
natural-gradient work supplies the broader background on invariant steepest
descent \cite{AmariNG,amari2016information,martens2020new}.

\paragraph{Path dilation in self-concordant Hessian geometry.}
Nesterov and Todd work in the same broad class of Hessian Riemannian metrics.
Their Definition~3.3 calls a curve $\tau$-geodesic when every subarc has length
at most $\tau$ times the Riemannian distance between its endpoints; their
Theorem~5.2 proves that a primal-dual central path is $\sqrt{2}$-geodesic
\cite{nesterov2002riemannian}. Corollary~3.2, Lemma~3.4, and Theorem~3.1 also
connect dilation with local-norm short-step complexity. Their barrier
structure controls specially generated central paths. We fix the mixture
segment and seek exact worst-case constants or rates over statistical
families; our bounded-regime proof instead compares a fixed
expectation-coordinate ruler over the whole finite-diameter manifold. In a
complementary direction, Nielsen bounds Fisher--Rao distance in Hessian models
by the energy of dual affine paths
\cite[Propositions~7--8]{nielsen2024approximation}; those endpoint upper bounds
do not control the mixture path's length-to-distance dilation.

\paragraph{Thermodynamic length and optimization trajectories.}
Salamon and Berry introduced thermodynamic length as a lower-bound device for
dissipated availability \cite{salamon1983thermo}. In linear response, Sivak
and Crooks derive a friction-tensor metric, prove the work bound
$W_{\mathrm{ex}}\ge L^2/\Delta t$, and identify geodesics with
minimum-dissipation protocols \cite{sivak2012thermodynamic}; their metric
reduces to Fisher information only under additional relaxation-time
assumptions. Bravetti, Garc\'ia Ariza, and Padilla connect dual-pregeodesic
relaxation to a Horse--Carrot bound and use the Amari--Chentsov tensor to order
directional relaxation asymmetry \cite{bravetti2025asymmetric}. These works
support the general principle that path length becomes operational under a
per-change budget, but use a different cost from our Fisher-chord accounting.
Gupta, Balakrishnan, and Ramdas study Euclidean lengths of gradient-descent and
gradient-flow trajectories, including logarithmic regimes
\cite{gupta2021path}; their numerator is an algorithmic path in an ambient
Euclidean metric, rather than an affine-connection geodesic normalized by
Fisher--Rao distance. None of these results gives our categorical constant,
scale-product rate, or forward-KL segment-invariance statement.

\paragraph{Categorical Fisher geometry.}
For categorical distributions, the square-root embedding of the open simplex
into a radius-two sphere and the formula
$d_{\mathrm{FR}}(p,q)=2\arccos\!\left(\sum_i\sqrt{p_iq_i}\right)$ are standard;
Miyamoto et al. give a modern treatment, including the distinction between the
open simplex and its boundary extension
\cite[Section~3.5]{miyamoto2024closedform}. Our categorical result begins with
the numerator: an arcsine-weight calculation of the mixture segment's Fisher
length yields the sharp worst-case factor $\pi/(2\sqrt{2})$.

\paragraph{Gaussian covariance geometry and scale products.}
For zero-mean Gaussians, Wells et al. establish the covariance-cone Fisher
metric and distance
\[
g_\Sigma(A,B)=\tfrac12\operatorname{tr}(\Sigma^{-1}A\Sigma^{-1}B),
\qquad
d_{\mathrm{FR}}(\Sigma_0,\Sigma_1)
=\tfrac1{\sqrt2}\bigl\|\log(\Sigma_0^{-1/2}\Sigma_1
\Sigma_0^{-1/2})\bigr\|_F.
\]
These results \cite{wells2020covariance} supply our denominator and its
normalization, not the length of the covariance-linear mixture path. We also
keep distinct the covariance-linear $m$-path, the affine-invariant Fisher
geodesic, and the log-Euclidean interpolation of Arsigny et al., which belongs
to a different metric geometry \cite{arsigny2007geometric}. Our scale-product
analysis asks how reciprocal-affine channels accumulated across scales
separate the first two paths, producing the $\Theta(\!\sqrt{\log n})$ rate.

\paragraph{Trust regions and operational meaning.}
Nesterov and Todd's local-norm analysis is the direct antecedent for the
principle that, under a uniform local budget, step count tracks path length
\cite[Corollary~3.2, Lemma~3.4, and Theorem~3.1]{nesterov2002riemannian}.
TRPO motivates the old-to-new average-KL orientation of a statistical trust
region \cite{schulman2015trust}, but its implemented linearization, quadratic
KL model, conjugate-gradient solve, and line search do not satisfy the exact
KKT identity used here. Our narrower result combines exact forward-KL segment
invariance for an ideal exponential-family subproblem with a conditional,
one-sided shadowing lower bound under the Fisher-chord budget actually proved.

Accordingly, we do not claim to introduce length-to-distance ratios or the
dual-geodesic flow. We study the Fisher dilation of the distinguished mixture
geodesic, isolate tensor and global geometric mechanisms governing it, and
derive family-specific exact constants and asymptotic rates.

\section{Geometry Glossary}
\label{apx:sec:glossary}

This glossary fixes the differential-geometry vocabulary used in the proof
supplement. It is meant as a compact reference for readers who know linear
algebra and multivariable calculus but may not have studied Riemannian or
information geometry.

\subsection{Manifolds and tangent objects.}
\begin{description}
\item[\textbf{Manifold / smooth manifold.}]
A smooth manifold $M$ is a space that may be globally curved or constrained,
but locally looks like an open set of $\mathbb{R}^n$ with smooth coordinate changes.
In this paper, points of $M$ are probability distributions in a statistical
model.

\item[\textbf{Chart / coordinate chart.}]
A chart is a local coordinate map from a piece of $M$ to $\mathbb{R}^n$. The natural
parameter $\theta$ and expectation parameter $\eta$ are two different charts
on the same statistical manifold.

\item[\textbf{Dimension.}]
The dimension of $M$ is the number of local coordinates needed to describe a
point. Equivalently, it is the dimension of each tangent space $T_pM$.

\item[\textbf{Tangent space $T_pM$.}]
The tangent space at $p$ is the vector space of infinitesimal directions in
which one can move away from $p$. Formally, tangent vectors can be viewed as
directional-derivative operators acting on smooth scalar functions near $p$.

\item[\textbf{Tangent vector / velocity vector $\dot\gamma$.}]
If $\gamma(t)$ is a curve in $M$, then $\dot\gamma(t)\in T_{\gamma(t)}M$ is
its instantaneous velocity. It is the tangent vector obtained by differentiating
the curve at time $t$.

\item[\textbf{Vector field.}]
A vector field $X$ assigns a tangent vector $X_p\in T_pM$ smoothly to each
point $p$. Connections are needed because differentiating such an assignment
requires comparing vectors that live in different tangent spaces.

\item[\textbf{Covector / dual vector.}]
A covector at $p$ is a linear map $T_pM\to\mathbb{R}$. When the text says a covector is
determined by its pairings with all vectors, it means that if two covectors
give the same number on every tangent vector, they are the same covector.
\end{description}

\subsection{Metrics and lengths.}
\begin{description}
\item[\textbf{Riemannian metric $g$.}]
A Riemannian metric assigns an inner product $g_p(\cdot,\cdot)$ to each tangent
space $T_pM$, varying smoothly with $p$. Thus $g(X,Y)$ means ``take the inner
product of the two tangent vectors using the metric at the current point.''

\item[\textbf{Fisher metric.}]
For an exponential family, the Fisher metric is the Hessian of the log-partition
function: $g_{ij}=\partial_i\partial_j\psi$ in the natural chart. This is the
Riemannian metric used throughout the paper.

\item[\textbf{Hessian.}]
The Hessian of a scalar function is its matrix of second partial derivatives.
Thus $g=\nabla^2\psi$ means $g_{ij}=\partial_i\partial_j\psi$; here $\nabla$ is
the ordinary Euclidean gradient in the coordinate chart, not a connection.

\item[\textbf{Norm $\|v\|_x$.}]
The Fisher norm of a vector $v\in T_xM$ is
$\|v\|_x=\sqrt{g_x(v,v)}$. It is the local length assigned to the infinitesimal
move $v$.

\item[\textbf{Nondegenerate.}]
The metric is nondegenerate if $g(v,w)=0$ for every $w$ implies $v=0$. This is
the inner-product analogue of an invertible matrix: no nonzero vector is
invisible to all pairings.
\end{description}

\subsection{Connections and covariant differentiation.}
\begin{description}
\item[\textbf{Affine connection $\nabla$.}]
An affine connection is a rule for differentiating vector fields along vector
fields. It takes two vector fields $X,Y$ and produces another vector field
$\nabla_XY$, read as ``the derivative of $Y$ in direction $X$.''

\item[\textbf{Covariant derivative.}]
The covariant derivative $\nabla_XY$ measures the rate of change of the vector
field $Y$ along the direction $X$, after using the connection to compare nearby
tangent spaces. This is the coordinate-independent replacement for ordinary
componentwise differentiation of vector fields.

\item[\textbf{Christoffel symbols $\Gamma^k_{ij}$.}]
In a chart, Christoffel symbols are the coefficients of the connection:
$\nabla_{\partial_i}\partial_j=\Gamma^k_{ij}\partial_k$. The indices $i,j$
choose the derivative direction and the basis vector being differentiated;
the index $k$ is the output component, summed when repeated.

\item[\textbf{Levi-Civita connection $\nabla^{(\mathrm{LC})}$.}]
The Levi-Civita connection is the unique connection that is both torsion-free
and metric-compatible for $g$. Its geodesics are the local shortest paths for
the Fisher metric.

\item[\textbf{Torsion-free.}]
A connection is torsion-free if
$\nabla_XY-\nabla_YX=[X,Y]$, where $[X,Y]$ is the Lie bracket. In coordinate
bases this corresponds to symmetry of Christoffel symbols in the two lower
indices: $\Gamma^k_{ij}=\Gamma^k_{ji}$.

\item[\textbf{Metric-compatible.}]
Metric compatibility means
$X\,g(Y,Z)=g(\nabla_XY,Z)+g(Y,\nabla_XZ)$. It says parallel transport under the
connection preserves inner products measured by $g$.

\item[\textbf{Flat connection.}]
A connection is flat if its curvature is zero. Equivalently, locally there are
coordinates in which its Christoffel symbols vanish; this is why
$\Gamma^{(m)}=0$ in $\eta$-coordinates and $\Gamma^{(e)}=0$ in
$\theta$-coordinates.

\item[\textbf{The notation $\nabla^{(m)}_X$, $\nabla^{(e)}_X$,
and $\nabla^{(\mathrm{LC})}_X$.}]
These mean covariant differentiation in the direction $X$ using, respectively,
the mixture connection, exponential connection, and Levi-Civita connection.
They are three different notions of ``differentiate a vector field'' on the
same manifold.
\end{description}

\subsection{Dually flat statistical manifolds.}
\begin{description}
\item[\textbf{Dually flat.}]
A statistical manifold is dually flat when it has two flat, torsion-free
connections $\nabla^{(m)}$ and $\nabla^{(e)}$ that are dual with respect to
$g$:
\[
X\,g(Y,Z)=g(\nabla^{(m)}_XY,Z)+g(Y,\nabla^{(e)}_XZ).
\]
The two connections define two different straight-line geometries on $M$.

\item[\textbf{Natural parameter $\theta$.}]
The natural parameter is the coordinate of an exponential family in which
$\nabla^{(e)}$ is flat. In $\theta$-coordinates, $e$-geodesics are ordinary
affine line segments.

\item[\textbf{Expectation parameter $\eta$.}]
The expectation parameter is $\eta=\mathbb{E}_\theta[t(X)]$, the expected sufficient
statistic. It is the coordinate in which $\nabla^{(m)}$ is flat, so
$m$-geodesics are ordinary affine line segments in $\eta$.

\item[\textbf{Legendre duality.}]
For the convex log-partition function $\psi$, the dual coordinate is
$\eta=\nabla\psi(\theta)$ and the inverse relation is
$\theta=\nabla\psi^\ast(\eta)$. Here $\nabla$ is the ordinary coordinate
gradient, and $\psi^\ast$ is the Legendre dual of $\psi$.

\item[\textbf{Log-partition function $\psi$.}]
In an exponential family
$p_\theta(x)=\exp(\langle\theta,t(x)\rangle-\psi(\theta))$, $\psi$ is the
normalizing cumulant-generating function. Its derivatives encode moments of
the sufficient statistic.

\item[\textbf{Exponential family.}]
An exponential family is a model of the form
$p_\theta(x)=\exp(\langle\theta,t(x)\rangle-\psi(\theta))$. The vector
$t(x)$ is the sufficient statistic and $\theta$ is its natural parameter.

\item[\textbf{Bregman divergence.}]
The Bregman divergence generated by $\psi$ is
$D_\psi(\theta,\theta')=\psi(\theta)-\psi(\theta')-
\langle\nabla\psi(\theta'),\theta-\theta'\rangle$. For exponential families,
KL divergences are Bregman divergences in the natural or dual coordinates,
depending on orientation.
\end{description}

\subsection{Geodesics and pregeodesics.}
\begin{description}
\item[\textbf{Autoparallel curve / geodesic.}]
A curve $\gamma$ is a geodesic for a connection $\nabla$ if
$\nabla_{\dot\gamma}\dot\gamma=0$. This means the velocity is transported
parallel to itself: zero covariant acceleration.

\item[\textbf{Pregeodesic.}]
A pregeodesic is a curve whose image is the image of some geodesic, possibly
with a different speed. Equivalently,
$\nabla_{\dot\gamma}\dot\gamma=\widetilde\mu\,\dot\gamma$ for some scalar
$\widetilde\mu$ along the curve: acceleration is allowed, but only parallel to
the velocity.

\item[\textbf{Geodesic vs. pregeodesic.}]
A geodesic has the right path and the right speed for the chosen connection.
A pregeodesic has the right path but may use a non-affine speed.

\item[\textbf{Reparametrization.}]
Reparametrizing a curve means replacing $\gamma(s)$ by $\gamma(s(t))$ for a
monotone scalar function $s(t)$. This changes the speed along the curve but not
the set of points traced by the curve.

\item[\textbf{$m$-geodesic.}]
An $m$-geodesic is a geodesic of $\nabla^{(m)}$. Because $\nabla^{(m)}$ is flat
in $\eta$-coordinates, the $m$-geodesic from $p$ to $q$ is
$\eta(t)=(1-t)\eta_p+t\eta_q$.

\item[\textbf{$e$-geodesic.}]
An $e$-geodesic is a geodesic of $\nabla^{(e)}$. Because $\nabla^{(e)}$ is flat
in $\theta$-coordinates, it is a straight affine segment in $\theta$.

\item[\textbf{Levi-Civita geodesic.}]
A Levi-Civita geodesic is a geodesic of $\nabla^{(\mathrm{LC})}$. Locally, these are
the shortest paths for the Fisher metric.

\item[\textbf{Geodesic / Riemannian distance.}]
The distance $\delta_{\mathrm{geo}}(p,q)$ is the infimum of Fisher lengths of all $C^1$ curves
from $p$ to $q$. When the metric is the Fisher metric, this is also called the
Fisher--Rao distance.

\item[\textbf{Affine segment.}]
An affine segment is an ordinary straight line in a flat coordinate chart, such
as $(1-t)\eta_p+t\eta_q$ or $(1-t)\theta_p+t\theta_q$. The adjective ``affine''
refers to straightness in that coordinate system.
\end{description}

\subsection{Tensors.}
\begin{description}
\item[\textbf{$(p,q)$-tensor.}]
A $(p,q)$-tensor is a multilinear object with $p$ covector slots and $q$ vector
slots. The only operational rule needed here is multilinearity in each slot.

\item[\textbf{$(0,3)$-tensor.}]
A $(0,3)$-tensor takes three tangent vectors and returns a scalar, linearly in
each argument. Thus $T(X,Y,Z)$ is trilinear.

\item[\textbf{Fully symmetric tensor.}]
A tensor $T$ is fully symmetric if $T(X,Y,Z)$ is unchanged under any permutation
of the arguments. For example, $T(X,Y,Z)=T(Y,X,Z)=T(Z,Y,X)$.

\item[\textbf{Amari--Chentsov tensor $T$.}]
The Amari--Chentsov tensor measures the gap between the exponential and mixture
connections. Lowered with $g$,
$T(X,Y,Z)=g(\nabla^{(e)}_XY-\nabla^{(m)}_XY,Z)$; in the natural chart,
$T_{ijk}=-\partial_i\partial_j\partial_k\psi$.

\item[\textbf{Lowering an index.}]
Lowering an index means using the metric to convert a vector-valued quantity
into a covector-valued one. In this document, the vector
$\nabla^{(e)}_XY-\nabla^{(m)}_XY$ is converted to the scalar trilinear form
$T(X,Y,Z)$ by pairing it with $Z$ via $g$.
\end{description}

\subsection{Paper-specific quantities.}
\begin{description}
\item[\textbf{$R(p,q)$.}]
$R(p,q)=L(\gamma^{(m)})/\delta_{\mathrm{geo}}(p,q)\ge1$ is the ratio between the Fisher length
of the $m$-geodesic and the shortest Fisher--Rao distance from $p$ to $q$.
It measures the inefficiency of following the mixture-straight path.

\item[\textbf{$R^{\ast}(M)$.}]
$R^{\ast}(M)=\sup_{p\ne q}R(p,q)$ is the worst-case inefficiency over the
manifold. The paper's classes are organized by how this supremum behaves.

\item[\textbf{Skewness $\kappa_T$.}]
$\kappa_T$ is a normalized bound on the size of the Amari--Chentsov tensor, of the
form $\|T(v,v,\cdot)\|_g/\|v\|^3$. It controls how quickly the Fisher ruler
distorts along $m$-geodesics.

\item[\textbf{Fisher--Rao diameter $D(M)$.}]
$D(M)=\sup_{p,q}\delta_{\mathrm{geo}}(p,q)$ is the largest Fisher--Rao distance between two
points of $M$, possibly infinite.

\item[\textbf{Fisher-chord budget.}]
A Fisher-chord budget limits each natural-gradient step by the Fisher norm of
its $\eta$-coordinate displacement. It is the per-step local distance budget
used in the operational step-count result.

\item[\textbf{$\beta$-consistent budgeted scheme.}]
A $\beta$-consistent budgeted scheme is a discrete optimizer whose steps obey
the Fisher-chord budget and whose local error from the ideal $m$-geodesic flow
is second order in the step size. This condition lets the shadowing argument
compare discrete steps to the continuous $m$-geodesic.
\end{description}

\subsection{Global geometry and inequalities.}
\begin{description}
\item[\textbf{Hadamard manifold.}]
A Hadamard manifold is complete, simply connected, and has non-positive
sectional curvature. Its geodesics are globally well behaved; in particular,
between two points there is a unique globally minimizing geodesic.

\item[\textbf{Strongly convex normal neighborhood.}]
This is a neighborhood in which any two points are connected by a unique
minimizing Levi-Civita geodesic that remains inside the neighborhood. It is the
local region where geodesic coordinates and distance comparisons behave
unambiguously.

\item[\textbf{Curvature / Riemann curvature.}]
Curvature measures the failure of parallel transport around an infinitesimal
loop to return a vector to its original orientation. Sectional curvature is the
curvature seen in a two-dimensional tangent plane.

\item[\textbf{Riemann-flat.}]
A Riemann-flat manifold has zero Riemann curvature tensor, so it is locally
isometric to Euclidean space. This is a statement about the Levi-Civita
geometry, not merely about one connection's Christoffel symbols in one chart.

\item[\textbf{Bochner integral.}]
A Bochner integral is an integral whose values lie in a vector space, here a
fixed tangent space. In this document it lets us integrate a curve of tangent
vectors after they have been identified in one tangent space.

\item[\textbf{Cauchy--Schwarz in $(T_xM,g_x)$.}]
At each point $x$, the tangent space with inner product $g_x$ is an ordinary
inner-product space. Hence $|g_x(u,v)|\le\|u\|_x\|v\|_x$.
\end{description}


\begin{thebibliography}{33}
\providecommand{\natexlab}[1]{#1}

\bibitem[{Amari(1998)}]{AmariNG}
Amari, S.-i. 1998.
\newblock Natural Gradient Works Efficiently in Learning.
\newblock \emph{Neural Computation}, 10(2): 251--276.

\bibitem[{Amari(2016)}]{amari2016information}
Amari, S.-i. 2016.
\newblock \emph{Information geometry and its applications}.
\newblock Springer.

\bibitem[{Arsigny et~al.(2007)Arsigny, Fillard, Pennec, and
  Ayache}]{arsigny2007geometric}
Arsigny, V.; Fillard, P.; Pennec, X.; and Ayache, N. 2007.
\newblock Geometric means in a novel vector space structure on symmetric
  positive-definite matrices.
\newblock \emph{SIAM journal on matrix analysis and applications}, 29(1):
  328--347.

\bibitem[{Banach(1938)}]{banach1938homogene}
Banach, S. 1938.
\newblock {\"U}ber homogene {P}olynome in {$(L^{2})$}.
\newblock \emph{Studia Mathematica}, 7: 36--44.

\bibitem[{Bhatia(2009)}]{bhatia2009positive}
Bhatia, R. 2009.
\newblock Positive definite matrices.
\newblock In \emph{Positive Definite Matrices}. Princeton university press.

\bibitem[{Bravetti, Garc{\'\i}a~Ariza, and
  Padilla(2025)}]{bravetti2025asymmetric}
Bravetti, A.; Garc{\'\i}a~Ariza, M.~{\'A}.; and Padilla, P. 2025.
\newblock Asymmetric relaxations through the lens of information geometry.
\newblock \emph{Journal of Physics A: Mathematical and Theoretical}, 58(12):
  125004.

\bibitem[{Bravetti, Garc{\'\i}a~Ariza, and
  Romero-Arias(2026)}]{bravetti2026gradient}
Bravetti, A.; Garc{\'\i}a~Ariza, M.~{\'A}.; and Romero-Arias, J.~R. 2026.
\newblock Gradient systems and asymmetric relaxations in view of Riemannian
  geometry.
\newblock \emph{Entropy}, 28(5): 516.

\bibitem[{Dauphin et~al.(2014)Dauphin, Pascanu, Gulcehre, Cho, Ganguli, and
  Bengio}]{dauphin2014identifying}
Dauphin, Y.~N.; Pascanu, R.; Gulcehre, C.; Cho, K.; Ganguli, S.; and Bengio, Y.
  2014.
\newblock Identifying and attacking the saddle point problem in
  high-dimensional non-convex optimization.
\newblock \emph{Advances in neural information processing systems}, 27.

\bibitem[{Felice and Ay(2018)}]{felice2018dynamical}
Felice, D.; and Ay, N. 2018.
\newblock Dynamical Systems Induced by Canonical Divergence in Dually Flat
  Manifolds.
\newblock arXiv:1812.04461.

\bibitem[{Fujiwara and Amari(1995)}]{fujiwara1995gradient}
Fujiwara, A.; and Amari, S.-i. 1995.
\newblock Gradient systems in view of information geometry.
\newblock \emph{Physica D: Nonlinear Phenomena}, 80(3): 317--327.

\bibitem[{Gupta, Balakrishnan, and Ramdas(2021)}]{gupta2021path}
Gupta, C.; Balakrishnan, S.; and Ramdas, A. 2021.
\newblock Path length bounds for gradient descent and flow.
\newblock \emph{Journal of Machine Learning Research}, 22(68): 1--63.

\bibitem[{ichi Amari and Nagaoka(2000)}]{Amari2000MethodsOI}
ichi Amari, S.; and Nagaoka, H. 2000.
\newblock \emph{Methods of Information Geometry}, volume 191 of
  \emph{Translations of Mathematical Monographs}.
\newblock American Mathematical Society and Oxford University Press.

\bibitem[{Kakade(2001)}]{kakade2001natural}
Kakade, S.~M. 2001.
\newblock A natural policy gradient.
\newblock \emph{Advances in neural information processing systems}, 14.

\bibitem[{Khan et~al.(2018)Khan, Nielsen, Tangkaratt, Lin, Gal, and
  Srivastava}]{khan2018fast}
Khan, M.; Nielsen, D.; Tangkaratt, V.; Lin, W.; Gal, Y.; and Srivastava, A.
  2018.
\newblock Fast and scalable bayesian deep learning by weight-perturbation in
  adam.
\newblock In \emph{International conference on machine learning}, 2611--2620.
  PMLR.

\bibitem[{Kunstner, Balles, and Hennig(2019)}]{kunstner2019limitations}
Kunstner, F.; Balles, L.; and Hennig, P. 2019.
\newblock Limitations of the empirical fisher approximation for natural
  gradient descent.
\newblock \emph{Advances in neural information processing systems}, 32.

\bibitem[{Martens(2020)}]{martens2020new}
Martens, J. 2020.
\newblock New insights and perspectives on the natural gradient method.
\newblock \emph{Journal of Machine Learning Research}, 21(146): 1--76.

\bibitem[{Miyamoto et~al.(2024)Miyamoto, Meneghetti, Pinele, and
  Costa}]{miyamoto2024closedform}
Miyamoto, H.~K.; Meneghetti, F. C.~C.; Pinele, J.; and Costa, S. I.~R. 2024.
\newblock On Closed-Form Expressions for the Fisher--Rao Distance.
\newblock \emph{Information Geometry}, 7: 311--354.

\bibitem[{Nesterov and Nemirovskii(1994)}]{nesterov1994interior}
Nesterov, Y.; and Nemirovskii, A. 1994.
\newblock \emph{Interior-Point Polynomial Algorithms in Convex Programming}.
\newblock Philadelphia: SIAM.

\bibitem[{Nesterov and Todd(2002)}]{nesterov2002riemannian}
Nesterov, Y.~E.; and Todd, M.~J. 2002.
\newblock On the {R}iemannian Geometry Defined by Self-Concordant Barriers and
  Interior-Point Methods.
\newblock \emph{Foundations of Computational Mathematics}, 2: 333--361.

\bibitem[{Nielsen(2020)}]{nielsen2020elementary}
Nielsen, F. 2020.
\newblock An elementary introduction to information geometry.
\newblock \emph{Entropy}, 22(10): 1100.

\bibitem[{Nielsen(2021)}]{nielsen2021geodesic}
Nielsen, F. 2021.
\newblock On geodesic triangles with right angles in a dually flat space.
\newblock In \emph{Progress in information geometry: theory and applications},
  153--190. Springer.

\bibitem[{Nielsen(2023)}]{nielsen2023simple}
Nielsen, F. 2023.
\newblock A simple approximation method for the Fisher--Rao distance between
  multivariate normal distributions.
\newblock \emph{Entropy}, 25(4): 654.

\bibitem[{Nielsen(2024)}]{nielsen2024approximation}
Nielsen, F. 2024.
\newblock Approximation and bounding techniques for the Fisher-Rao distances
  between parametric statistical models.
\newblock In \emph{Handbook of Statistics}, volume~51, 67--116. Elsevier.

\bibitem[{Omiya and Komaki(2025)}]{omiya2025minimization}
Omiya, G.; and Komaki, F. 2025.
\newblock Minimization of Functions on Dually Flat Spaces Using Geodesic
  Descent Based on Dual Connections.
\newblock \emph{arXiv preprint arXiv:2512.09358}.

\bibitem[{Pennec(2020)}]{pennec2020manifold}
Pennec, X. 2020.
\newblock Manifold-valued image processing with SPD matrices.
\newblock In \emph{Riemannian geometric statistics in medical image analysis},
  75--134. Elsevier.

\bibitem[{Raskutti and Mukherjee(2015)}]{raskutti2015information}
Raskutti, G.; and Mukherjee, S. 2015.
\newblock The Information Geometry of Mirror Descent.
\newblock \emph{IEEE Transactions on Information Theory}, 61(3): 1451--1457.

\bibitem[{Salamon and Berry(1983)}]{salamon1983thermo}
Salamon, P.; and Berry, R. 1983.
\newblock Thermodynamic Length and Dissipated Availability.
\newblock \emph{Physical Review Letters - PHYS REV LETT}, 51: 1127--1130.

\bibitem[{Schulman et~al.(2015)Schulman, Levine, Abbeel, Jordan, and
  Moritz}]{schulman2015trust}
Schulman, J.; Levine, S.; Abbeel, P.; Jordan, M.; and Moritz, P. 2015.
\newblock Trust region policy optimization.
\newblock In \emph{International conference on machine learning}, 1889--1897.
  PMLR.

\bibitem[{Shima(2007)}]{shima2007geometry}
Shima, H. 2007.
\newblock \emph{The geometry of Hessian structures}.
\newblock World Scientific.

\bibitem[{Sivak and Crooks(2012)}]{sivak2012thermodynamic}
Sivak, D.~A.; and Crooks, G.~E. 2012.
\newblock Thermodynamic Metrics and Optimal Paths.
\newblock \emph{Physical Review Letters}, 108(19): 190602.

\bibitem[{Skovgaard(1984)}]{skovgaard1984riemannian}
Skovgaard, L.~T. 1984.
\newblock A Riemannian geometry of the multivariate normal model.
\newblock \emph{Scandinavian journal of statistics}, 211--223.

\bibitem[{Wells et~al.(2020)Wells, Cook, Pine, and
  Robinson}]{wells2020covariance}
Wells, J.; Cook, M.; Pine, K.; and Robinson, B.~D. 2020.
\newblock Fisher--Rao Distance on the Covariance Cone.
\newblock arXiv:2010.15861.

\bibitem[{Wierstra et~al.(2014)Wierstra, Schaul, Glasmachers, Sun, Peters, and
  Schmidhuber}]{wierstra2014natural}
Wierstra, D.; Schaul, T.; Glasmachers, T.; Sun, Y.; Peters, J.; and
  Schmidhuber, J. 2014.
\newblock Natural evolution strategies.
\newblock \emph{The Journal of Machine Learning Research}, 15(1): 949--980.

\end{thebibliography}
\end{document}